\documentclass{article}

\usepackage{microtype}
\usepackage{graphicx}
\usepackage{subfigure}
\usepackage{booktabs} 

\usepackage{hyperref}

\newcommand{\theHalgorithm}{\arabic{algorithm}}

\usepackage[accepted]{icml2023}

\usepackage{amsmath}
\usepackage{amssymb}
\usepackage{mathtools}
\usepackage{amsthm}

\usepackage[capitalize,noabbrev]{cleveref}

\usepackage{comment}
\usepackage{bbm}
\usepackage{subfigmat}

\theoremstyle{plain}
\newtheorem{theorem}{Theorem}[section]
\newtheorem{proposition}[theorem]{Proposition}
\newtheorem{lemma}[theorem]{Lemma}

\theoremstyle{definition}
\newtheorem{definition}[theorem]{Definition}
\newtheorem{assumption}[theorem]{Assumption}
\theoremstyle{remark}
\newtheorem{remark}[theorem]{Remark}

\usepackage[textsize=tiny]{todonotes}

\icmltitlerunning{DIFF2: Differential Private Optimization via Gradient Differences}

\begin{document}

\twocolumn[
\icmltitle{DIFF2: Differential Private Optimization via Gradient Differences \\ for Nonconvex Distributed Learning}



\icmlsetsymbol{equal}{*}

\begin{icmlauthorlist}
\icmlauthor{Tomoya Murata}{msi,tokyo}
\icmlauthor{Taiji Suzuki}{tokyo,riken}
\end{icmlauthorlist}

\icmlaffiliation{msi}{NTT DATA Mathematical Systems Inc., Tokyo, Japan}
\icmlaffiliation{tokyo}{Graduate School of Information Science and Technology, The University of Tokyo, Tokyo, Japan}
\icmlaffiliation{riken}{Center for Advanced Intelligence Project, RIKEN, Tokyo, Japan}

\icmlcorrespondingauthor{Tomoya Murata}{murata@msi.co.jp}
\icmlcorrespondingauthor{Taiji Suzuki}{taiji@mist.i.u-tokyo.ac.jp}

\icmlkeywords{Differential Privacy, Nonconvex Optimization, Distributed Learning, Utility Improvement}

\vskip 0.3in
]



\printAffiliationsAndNotice{\icmlEqualContribution} 

\begin{abstract}
Differential private optimization for nonconvex smooth objective is considered. In the previous work, the best known utility bound is $\widetilde O(\sqrt{d}/(n\varepsilon_\mathrm{DP}))$ in terms of the squared full gradient norm, which is achieved by Differential Private Gradient Descent (DP-GD) as an instance, where $n$ is the sample size, $d$ is the problem dimensionality and $\varepsilon_\mathrm{DP}$ is the differential privacy parameter. To improve the best known utility bound, we propose a new differential private optimization framework called \emph{DIFF2 (DIFFerential private optimization via gradient DIFFerences)} that constructs a differential private global gradient estimator with possibly quite small variance based on communicated \emph{gradient differences} rather than gradients themselves. It is shown that DIFF2 with a gradient descent subroutine achieves the utility of $\widetilde O(d^{2/3}/(n\varepsilon_\mathrm{DP})^{4/3})$, which can be significantly better than the previous one in terms of the dependence on the sample size $n$. To the best of our knowledge, this is the first fundamental result to improve the standard utility $\widetilde O(\sqrt{d}/(n\varepsilon_\mathrm{DP}))$ for nonconvex objectives. Additionally, a more computational and communication efficient subroutine is combined with DIFF2 and its theoretical analysis is also given. Numerical experiments are conducted to validate the superiority of DIFF2 framework. 
\end{abstract}

\section{Introduction}\label{sec: intro}
Privacy protection has become one of the important components in modern artificial intelligence technologies. Datasets consisted of browsing history, check-in data, medical records, financial records and so on, often contain private information, and then publishing the statistics like trained models or computed gradients on the private datasets causes a leak of the private information. In the recently proposed federated learning framework \citep{konevcny2016federated, mcmahan2017communication}, where each client possesses a private local dataset and iteratively exchanges the own models or gradients with the central server, it is crucial to prevent the malicious clients from stealing the private information contained in the local datasets. 

Actually, many attacks to steal the private information from the trained models or gradients are known. Major attacks include membership inference attack \citep{shokri2017membership, yeom2018privacy, nasr2019comprehensive} and reconstruction attack \citep{fredrikson2015model, zhu2019deep, yang2019neural, wang2019beyond, geiping2020inverting, zhang2020secret}. Membership inference attack aims to infer the presence of the known individual or record in the dataset from the shared statistics of the dataset. Reconstruction attack aims to reconstruct the data samples of a victim user or client from the shared information like gradients. Due to the empirical success of these attacks, privacy protection technology is essentially required.

To preserve the privacy of the datasets, the simplest approach is data anonymization. Roughly speaking, data anonymization attempts to directly remove personally identifiable information from the datasets. The most major data anonymization technique is $k$-anonymization \citep{samarati1998protecting} and its variants \citep{li2006t, machanavajjhala2007diversity}. $k$-anonymization suppresses/generalizes attributes and/or adds dummy records to ensure that each record is similar to at least $k-1$ other records on the potentially identifying attributes called quasi-identifiers. However, $k$-anonymization is insufficient for privacy protection due to the existence of de-anonymization algorithms using auxiliary information. On the other side, the advanced variants of $k$-anonymization often destroy the utility of the dataset and makes the learned model degraded, although it improves the security of $k$-anonymization.

Another approach for privacy protection is the use of differential privacy technique. Differential privacy is a general concept to quantify the degree of privacy protection \citep{dwork2006calibrating, dwork2010difficulties, dwork2010boosting, dwork2011firm, dwork2014algorithmic}.
To guarantee the differential privacy of the machine learning algorithms, several approaches have been proposed. The output perturbation and objective perturbation are general differential private approaches for non-distributed environments. In distributed environments, a nice way to guarantee differential privacy is gradient perturbation, where some noise is added to the gradients before sharing them \citep{song2013stochastic, abadi2016deep, mcmahan2017learning, geyer2017differentially, triastcyn2019federated}. For example, in Differential Private Gradient Descent (DP-GD) algorithm, independent mean zero Gaussian noise is added to the gradient and the resulting noisy gradient is used for the parameter update at each iteration. 

In differential private optimization, it is crucial to compare the optimization accuracy called \emph{utility} of optimization algorithms satisfying pre-defined differential privacy level, since there is generally a trade-off relationship between the utility and the privacy level. Recent DP optimization studies have not only provided differential privacy guarantees, but also theoretically investigated the utility of the optimization algorithms. 

In the previous work, the best known utility bound for general nonconvex differential private optimization is $\widetilde O(\sqrt{d}/(n\varepsilon_\mathrm{DP}))$\footnotemark in terms of the squared global gradient norm, which is achieved by DP-GD as an instance, where $n$ is the sample size, $d$ is the problem dimensionality and $\varepsilon_\mathrm{DP}$ is the differential privacy parameter \citep{zhang2017efficient, wang2017differentially}. It is an important open problem whether it is possible or not to improve the best known utility bound in terms of the dependence on the sample size $n$, i.e., sample efficiency.
\footnotetext{$\widetilde O$, $\widetilde \Theta$ and $\widetilde \Omega$ symbols hide an extra poly-logarithmic factor depending on $1/\delta_\mathrm{DP}$ for simple presentation.}

\subsection*{Main contributions}
We develop \emph{DIFFerential optimizatrion via gradient DIFFerences (DIFF2)} framework to improve the previous utility bound for nonconvex DP optimization. The main features of DIFF2 are described as follows:\par

{\bf{Algorithmic Features. }}Main algorithmic features of DIFF2 framework are: (i) sharing \emph{local gradient differences} rather than gradients themselves to reduce the DP noise size; and (ii) constructing a DP global gradient estimator using \emph{the sum of the aggregated gradient difference and the previous DP global gradient estimator}. The obtained global gradient estimator is fed to a general sub-routine that optimizes the objective based on it. 
\par
{\bf{Theoretical Features. }}DIFF2 with a gradient descent subroutine achieves the utility of $\widetilde O(d^{2/3}/(n\varepsilon_\mathrm{DP})^{4/3})$, that can be \emph{significantly better than the best known utility} $\widetilde O(\sqrt{d}/(n\varepsilon_\mathrm{DP}))$ of DP-GD for nonconvex objectives in terms of the dependence on the sample size $n$. Our DP analysis relies on the fact that \emph{a gradient difference has possibly much smaller sensitivity than a gradient itself for smooth loss}. To derive minimum DP noise levels, R\'enyi differential privacy (RDP) technique is essentially used. In our utility analysis, it is crucial to carefully \emph{evaluate both the bias and variance of the DP global gradient estimator} generated by DIFF2 and \emph{determine the optimal choice of the restart interval} (see Subsection \ref{subsec: utility_analysis_diff2_gd} for the details). We further consider a more efficient subroutine called Bias-Variance Reduced Local SGD (BVR-L-SGD) routine and show that DIFF2 with BVR-L-SGD routine achieves better computational and communication efficiency than DIFF2 with gradient descent routine.  

From a theoretical point of view, we compare the best achievable utility, the stochastic gradient complexity and communication complexity to achieve the utility for several DP algorithms on nonconvex cases in Table \ref{tab: comp_to_sr_nsgd_and_dp_srm}. $N$ is defined as $\min_{p \in [P]}{n_p}P$. From this table, we can see that DIFF2-BVR-L-SGD achieves the best utility with the best gradient complexity and the best communication complexity in terms of the dependence on $N$. \footnotemark \footnotetext{Although the gradient complexity of DIFF2-BVR-L-SGD is a bit messy, we can check that the dependence on $N$ is strictly less than $N^2$. } \footnotemark \footnotetext{Note that achieving better utility $\varepsilon_\mathrm{opt}$ generally requires higher gradient complexity and communication complexity. Thus, comparing SRN-SGD, DP-SRM, and DIFF2-GD with DP-GD and DP-SGD is not fair and is biased in favor of DP-GD and DP-SGD. Notably, however, DIFF2-BVR-L-SGD still achieves better gradient complexity and communication complexity than these methods including DP-SGD. }

\begin{table*}[t]
    \centering
    \scalebox{0.95}{
    \begin{tabular}{c c c c}
        \hline
        Algorithm & Utility & Gradient Complexity & Communication Complexity\\ \hline \hline
        DP-GD & $\frac{\sqrt{d}}{N\varepsilon_\mathrm{DP}}$ &  $N + \frac{N^2\varepsilon_\mathrm{DP}}{\sqrt{d}}$ 
        & $\frac{\text{Gradient Complexity}}{N}$ \\
        DP-SGD & $\frac{\sqrt{d}}{N\varepsilon_\mathrm{DP}}$ & $\frac{N^2\varepsilon_\mathrm{DP}^2}{d}$ &  $1 + \frac{N^2\varepsilon_\mathrm{DP}^2}{b d}$\\
        \begin{tabular}{c}
        SRN-SGD  \\
        \citep{tran2022momentum}
        \end{tabular} & $\frac{d^\frac{2}{3}}{(N\varepsilon_\mathrm{DP})^\frac{4}{3}} + \frac{1}{N}$ & $\frac{N^\frac{7}{3}\varepsilon_\mathrm{DP}^\frac{4}{3}}{d^\frac{2}{3}}$ 
        & $\frac{N^\frac{7}{3}\varepsilon_\mathrm{DP}^\frac{4}{3}}{d^\frac{2}{3}}$\footnotemark\\
        \begin{tabular}{c}
        DP-SRM  \\
        \citep{wang2019efficient}
        \end{tabular}  & $\frac{d^\frac{2}{3}}{(N\varepsilon_\mathrm{DP})^\frac{4}{3}}$ & $\frac{N^2\varepsilon_\mathrm{DP}^2}{d}$
        & $1+ \frac{(N\varepsilon_\mathrm{DP})^\frac{4}{3}}{d^\frac{2}{3}}$ \\
        DIFF2-GD ({\color{red}this study}) & $\frac{d^\frac{2}{3}}{(N\varepsilon_\mathrm{DP})^\frac{4}{3}}$ & $N + \frac{N^\frac{7}{3}\varepsilon_\mathrm{DP}^\frac{4}{3}}{d^\frac{2}{3}}$ 
        & $\frac{\text{Gradient Complexity}}{N}$\\
        \begin{tabular}{c}
             DIFF2-BVR-L-SGD ({\color{red}this study})  \\
             ($K=b=\sqrt{N}$)
        \end{tabular} 
        & $\frac{d^\frac{2}{3}}{(N\varepsilon_\mathrm{DP})^\frac{4}{3}}$ & $N + \frac{N^\frac{11}{6}\varepsilon_\mathrm{DP}^\frac{4}{3}}{d^\frac{2}{3}}+ \frac{N^\frac{19}{12}\varepsilon_\mathrm{DP}^\frac{5}{6}}{d^\frac{1}{6}} + \frac{N^\frac{5}{3}P\varepsilon_\mathrm{DP}^\frac{2}{3}}{d^\frac{1}{3}}$
        & $\frac{\text{Gradient Complexity}}{N} + \frac{\zeta_2(N\varepsilon_\mathrm{DP})^\frac{4}{3}}{d^\frac{2}{3}}$\\
   \hline
   \end{tabular}
   }
   \caption{Comparison of the order of the utility and the stochastic gradient complexity of $(\varepsilon_\mathrm{DP}, \delta_\mathrm{DP})$-DP guaranteed optimization algorithms for nonconvex objectives. ``Utility" means the best achievable optimization error in terms of the squared norm of the gradient $\mathbb{E}\|\nabla f(x_\mathrm{out})\|^2$. ``Gradient Complexity" means the necessary number of single stochastic gradient evaluations to achieve the best utility. ``Communication Complexity" indicates the necessary number of communication rounds to achieve the best utility. $\zeta_2 = O(1)$ is the Hessian heterogeneity of the local datasets. $N$ is defined as $\min_{p \in [P]}{n_p}P$. $d$ is the dimensionality of the parameter space.  \citet{tran2022momentum} and \citet{wang2019efficient} (updated version in 2023) are independent works to this study. }  
   \label{tab: comp_to_sr_nsgd_and_dp_srm}
\end{table*}

\subsection*{Related work}
Here, we briefly review the related work to this paper. \par
{\bf{Convex Optimization. }}
\citet{chaudhuri2011differentially} has studied output perturbation and objective perturbation approaches, and investigated their utility for convex problems. Later, gradient perturbation has been proposed \citep{song2013stochastic} and its utility 
has been studied   \citep{bassily2014private, wang2017differentially, bassily2019private}. The obtained utility bounds in terms of the objective gap are $\widetilde O(\sqrt{d}/(n\varepsilon_\mathrm{DP}))$ for non-strongly convex objectives and $\widetilde O(d/(n^2\varepsilon_\mathrm{DP}^2))$ for strongly convex objectives to guarantee  $(\varepsilon_\mathrm{DP}, \delta_\mathrm{DP})$-DP. Several papers have derived the lower bounds that matches these upper bounds and shown the optimality of DP-(S)GD for $\ell_2$ bounded domains \citep{bassily2014private}, $\ell_1$ bounded domains \citep{asi2021private} and unconstrained domains \citep{liu2021tight}.
\par
{\bf{Nonconvex Optimization.}}
\citet{zhang2017efficient} has shown the utility $\widetilde O(\sqrt{d}/(n\varepsilon_\mathrm{DP}))$ of DP-SGD in terms of the squared gradient norm, i.e., first-order optimality, for noncovnex smooth objectives. DP-RMSprop and DP-Adam have been studied and the essentially same utility bound as DP-GD has been shown \citep{zhou2020private}. \citet{wang2017differentially} has given unified analysis of DP-GD and DP-SVRG for both convex and nonconvex cases and in particular utility $\widetilde O(d/(n^2\varepsilon_\mathrm{DP}^2))$ of DP-GD in terms of the objective gap has been shown for objectives satisfying Polyak-Łojasiewicz (PL) condition. \citet{wang2019differentially} has provided a new utility $\widetilde O(1/(n\varepsilon_\mathrm{DP}^2) + d/n^{\Omega(1)})$ of DP-GD in terms of the objective gap without PL condition based on the theory of gradient Langevin dynamics. Also, \citet{wang2019differentially} has provided a variant of DP-GD for finding second-order optimal points with the same utility as the first-order optimality cases. 
\par
{\bf{Computation and Communication Efficiency.}}
While computation and communication efficiency are not main focuses of this study, many works have considered these efficiencies for practical applications. DP-SGD is a nice candidate to reduce the expensive per update computational cost of DP-GD. For further reducing computational cost, applying variance reduction technique \citep{johnson2013accelerating, defazio2014saga, nguyen2017sarah, cutkosky2019momentum} has been studied \citep{wang2019efficient, bassily2019private, asi2021private}. Also, \citet{kuru2022differentially} has combined DP-GD with Nesterov's acceleration to improve the computational efficiency. On the other side, several works have developed more communication efficient DP algorithms than DP-GD in distributed learning. 
\citet{noble2022differentially} has shown that their proposed DP-FedAvg and DP-SCAFFOLD achieves better communication efficiency in terms of the number of communication rounds compared with DP-SGD, while maintaining the utility of DP-SGD. On the other side, reducing communication cost per update by gradient compression has been also investigated in DP optimization settings by several papers \citep{agarwal2018cpsgd, zhang2020private, ding2021differentially,  li2022soteriafl}. 
\par
{\bf{Utility Improvements.}}
Several works have attempted to improve the utility of DP-SGD in both theoretically and practically. \citet{chourasia2021differential} has shown the convergence of the privacy loss in DP-GD based on the theory of gradient Langevin dynamics for strongly convex objectives and pointed out that the previous DP-analysis relying on composition theorems may be loose. However, the obtained utility is the same as the best known one.  \citet{du2021dynamic} has considered dynamic scheduling of the DP noise size and the gradient clipping radius in DP-SGD to improve the utility and empirically justified its effectiveness. \citet{wang2019dp} has proposed a combination of DP-SGD with Laplacian smoothing technique to denoise the DP noise and reported its empirical out-performances. This approach has been further extended to the case of federated learning settings \citep{liang2020differentially}. Very recently, two independent works \citep{tran2022momentum, wang2019efficient}\footnotemark \footnotetext{ \citet{wang2019efficient} has been updated in 2023.} have improved this utility bound based on a similar idea to ours for non-distributed learning settings.  
\par
{\bf{Extensions of the Problem Settings .}}
\citet{wang2020differentially, hu2022high, kamath2022improved} have attempted to replace the uniform gradient boundedness assumption typically used in the DP guratantee analysis by the boundedness of the $k$-th moment of the gradient distribution to deal with heavy-tailed gradients for convex optimization. \citet{das2022beyond} has considered relaxing the uniform gradient boundedness assumption to sample dependent boundedness one. \citet{han2022differentially} has extended DP-SGD to the case of Riemannian optimization. Several works have developed differential private version of the alternating direction method of multipliers (ADMM) for centralized and even decentralized distributed learning \citep{huang2019dp, huang2020differentially,  shang2021differentially}. \citet{yang2021stability} has developed DP-SGD  theory for pairwise learning. 


\section{Problem Setting and Assumptions}
\label{sec: problem_setting}

In this section, we first introduce notation used in this paper. Then, the problem setting is illustrated and theoretical assumptions used in our analysis are given. \par

{\bf{Notation.}}
$\| \cdot \|$ denotes the Euclidean $L_2$ norm $\| \cdot \|_2$: $\|x\| = \sqrt{\sum_{i}x_i^2}$ for vector $x$. 
For a matrix $X$, $\|X\|$ denotes the induced norm by the Euclidean $L_2$ norm. For a natural number $m$, $[m]$ denotes the set $\{1, 2, \ldots, m\}$.
For a set $A$, $\# A$ means the number of elements, which is possibly $\infty$. For any number $a, b$, $a \vee b$ denotes $\mathrm{max}\{a, b\}$ and $a \wedge b$ does $\mathrm{min}\{a, b\}$. For a set $A$, we denote the uniform distribution over $A$ by $\mathrm{Unif}(A)$. 
\par

\subsection{Problem Setting}
We want to find an approximate minimizer of objective function $f(x) := (1/P)\sum_{p=1}^P f_p(x)$ with DP guarantee for every client $p \in [P]$ in centralized distributed learning settings, where $f_p(x) := (1/n_p)\sum_{i=1}^{n_p} \ell(x, z_i^{(p)})$ is the risk on private dataset $D_p$ of client $p$. \par 
\par 

{\bf{Threat Model.}} In this work, we focus on \emph{central differential privacy}. Central differential privacy assumes a trusted central server and honest-but-curious clients. In this case, when the server aggregates the vectors received from the clients and sends the result to each client, we only need to care the client's information leakage from the aggregated vector rather than the individual vector sent from a client. 

Then, it is required to guarantee \emph{record level} $(\varepsilon, \delta)$-differential privacy\footnotemark of outputs $\mathcal M(D_1, \ldots, D_P)$ from the central server with respect to each local dataset $D_p$, that is $\mathbb{P}(\mathcal M(\mathcal D) \in S)
    \leq e^\varepsilon \mathbb{P}(\mathcal M(\mathcal D') \in S) + \delta$,
where $\mathcal D := (D_1, \ldots, D_p, \ldots, D_P)$ and $\mathcal D' := (D_1, \ldots, D_p', \ldots, D_P)$, 
for every $S$ and every adjacent local datasets $D_p$ and $D_p'$. Here, we say that datasets $D = \{z_i\}_{i=1}^n$ and $D' = \{z_i'\}_{i=1}^n$ are adjacent if $d_\mathrm{H}(D, D') = 1$, where $d_\mathrm{H}$ is the Hamming distance between $D$ and $D'$ defined by $d_\mathrm{H}(D, D') := \sum_{i=1}^n \mathbbm{1}_{x_i \neq x_i'}$.   

\footnotetext{We focus on record-level differential privacy. It is straightforward to extend our analysis to the case of client-level differential privacy.}

When the central server is not trustworthy, one can use some secure aggregation technique, e.g., Multi-Party Computation (MPC), to guarantee the same privacy level as the trusted server case. Using secure aggregation technique, the server (and clients) can only access the aggregated result rather than the individual vectors sent from the clients. \par

{\bf{Theoretical Performance Measure.}} In this work, we focus on finding an approximate first-order stationary point, since it is generally difficult to find an approximate global minima or even local minima of $f$ due to the nonconvex nature of $f$. Then, given privacy parameters $\varepsilon_\mathrm{DP}$ and $\delta_\mathrm{DP}$, we measure $\|\nabla f(x)\|^2$ for the output $x$ of an $(\varepsilon_\mathrm{DP}, \delta_\mathrm{DP})$-DP algorithm. The smaller $\|\nabla f(x)\|^2$ indicates the higher utility. 

\subsection{Theoretical Assumptions}
Here, theoretical assumptions used in our analysis are described.

\begin{assumption}[Smoothness]
\label{assump: local_loss_smoothness}
For every $z \in \cup_{p=1}^P\mathrm{supp}(D_{p})$, $\ell(\cdot, z)$ is $L$-smooth, i.e., 
\begin{align*}
    \|\nabla \ell(x, z) - \nabla  \ell(y, z)\| \leq L\|x - y\|, \forall x, y \in \mathbb{R}^d. 
\end{align*}
\end{assumption}
We assume $L$-smoothness of loss $\ell$ rather than risk $f$, that is crucial to DIFF2 framework.

\begin{assumption}[Existence of global optimum]
\label{assump: optimal_sol}
$f$ has a global minimizer $x_* \in \mathbb{R}^{d}$.
\end{assumption}

\begin{assumption}[Gradient boundedness]
\label{assump: local_loss_gradient_boundedness}
For every $z \in \cup_{p=1}^P\mathrm{supp}(D_{p})$, $\ell(\cdot, z)$ is $G$-gradient bounded, i.e.,
\begin{align*}
    \|\nabla \ell(x, z)\| \leq G, \forall x \in \mathbb{R}^d. 
\end{align*}
\end{assumption}

Assumption \ref{assump: local_loss_gradient_boundedness} is a bit strong, particularly in FL, but is  necessary for our utility analysis and standard in the previous DP optimization literature. 

\section{Proposed Framework: DIFF2}
In this section, we provide the procedures of our proposed framework. 

\begin{algorithm}[t]
\caption{DIFF2($x_{0}$, $\sigma_1$, $\sigma_2$, $C_1$, $C_2$, $T$, $R$, $*\mathrm{args}$)}
\label{alg: DIFF2}
\begin{algorithmic}[1]
\STATE \# The central server is assumed to be trust-worthy. 
\FOR {$r=1$ to $R$}
    \FOR {client $p \in [P]$ in parallel}
        \IF {$(r - 1) \% T = 0$}
            \STATE $d_{r}^{(p)} = $ ClippedMean$(\{\nabla \ell(x_{r-1}, z)\}\}_{z \in D_p}, C_{1})\}$.
            \STATE Send $d_r^{(p)}$ to the central server. \label{line: diff2_grad_computation}
        \ELSE
            \STATE Send $d_{r}^{(p)} = $ ClippedMean$(\{\nabla \ell(x_{r-1}, z) - \nabla \ell(x_{r-2}, z)\}\}_{z \in D_p}, C_{2, r})\}$ to the central server, where $C_{2, r} = C_2\|x_{r-1} - x_{r-2}\|$. \label{line: diff2_grad_diff_computation}
        \ENDIF
    \ENDFOR
    \IF {$(r-1) \% T = 0$}
        \STATE Set $\sigma = \sigma_1$, $C = C_1$ and $\widetilde v_{r-1} = 0$. 
    \ELSE
        \STATE Set $\sigma = \sigma_2$ and $C = C_{2, r}$.
    \ENDIF
    \STATE Set $v_{r} = \frac{1}{P}\sum_{p=1}^P d_{r}^{(p)} + \widetilde v_{r-1}$ and $\widetilde v_{r} = v_{r} + \xi_{r}$, where $\xi_{r} \sim N(0, \sigma^2 C^2 I)$. \label{line: diff2_v_r_update}
    \STATE $x_r, x_r^\mathrm{out} =$ General-Routine($x_{r-1}$,  $\widetilde v_{r}$, $*\mathrm{args}$)
    \FOR {$p \in [P]$ in parallel}
        \STATE Receive $ x_{r}$ from the central server.
    \ENDFOR
\ENDFOR
\STATE {\bf{Return:}} $ x^{\mathrm{out}} =  x_{\hat r-1}^{\mathrm{out}}$ ($\hat r \sim \mathrm{Unif}[R])$).
\end{algorithmic}
\end{algorithm}

\begin{algorithm}[t]
\caption{ClippedMean($\{x_i\}_{i \in I}$, $C$)}
\label{alg: clipped_mean}
\begin{algorithmic}[1]
\STATE $\hat x_i = \min\left\{\frac{C}{\|x_i\|}, 1\right\}x_i$ for $i \in I$. 
\STATE {\bf{Return:}} $\frac{1}{|I|}\sum_{i \in I} \hat x_i$. 
\end{algorithmic}
\end{algorithm}

\begin{algorithm}[t]
\caption{GD-Routine($x_{r-1}$, $\widetilde v_{r}$, $\eta$)}
\label{alg: gd_routine}
\begin{algorithmic}[1]
\STATE Set $x_{r} = x_{r-1} - \eta \widetilde v_r$.
\STATE {\bf{Return:}} $x_r$, $x_r$. 
\end{algorithmic}
\end{algorithm}

{\bf{High-level Idea.}}
The main process of DIFF2 is to construct a differential private global gradient estimator $\widetilde v_r$ based on \emph{local gradient differences}. Recall that the standard differential private algorithms simply apply Gaussian mechanism to $\nabla f(x_{r-1})$, that is $\widetilde v_r := \nabla f(x_{r-1}) + \xi_r$, where $\xi_r$ is mean zero Gaussian noise. To DP guarantee, the standard deviation of $\xi_r$ should be proportional to the sensitivity of $\nabla f(x_{r-1})$, which is bounded by $2 G/(n_\mathrm{min}P)$ for $G$-Lipschitz loss function $\ell$. In contrast, DIFF2 uses an approximation
\begin{align*}
    \nabla f(x_{r-1}) =&\ \nabla f(x_{r-1}) - \nabla f(x_{r-2}) + \nabla f(x_{r-2}) \\
    \approx&\ \nabla f(x_{r-1}) - \nabla f(x_{r-2}) + \widetilde v_{r-1}
\end{align*}
and 
constructs new estimator $\widetilde v_r$ as follows:
\begin{align*}
    \widetilde v_r := (\nabla f(x_{r-1}) - \nabla f(x_{r-2}) + \xi_r) + \widetilde v_{r-1}, 
\end{align*}
with $\widetilde v_1 := \nabla f(x_0) + \xi_1$. Then, we observe that the standard deviation of $\xi_r$ depends on the sensitivity of the gradient difference $\nabla f(x_{r-1}) - \nabla f(x_{r-2})$, that is bounded by $L\|x_{r-1} - x_{r-2}\|$ for $L$-smoothness loss function $\ell$. Hence, if $x_{r-1}$ and $x_{r-2}$ are close, the sensitivity of the gradient difference can be smaller than the one of the gradient itself $\nabla f(x_{r-1})$, and the noise size of DIFF2 at round $r$ becomes smaller than the standard noise size. This is the mechanism of DIFF2 and explains why the gradient estimator of DIFF2 potentially has lower variance than the standard one. 
\par

{\bf{New Framework: DIFF2.}} The concrete procedures of our proposed framework DIFF2 is provided in Algorithm \ref{alg: DIFF2}. \par 

As explained above, DIFF2 constructs differential private global gradient estimator $\widetilde v_r$. Then, the estimator is fed to a general sub-routine that runs some (possibly local) optimization based on $\widetilde v_r$ (line 17).  

Now, we look at the concrete procedures of the construction of $\widetilde v_r$. At initial round, each worker computes the local gradients and applies ClippedMean (Algorithm \ref{alg: clipped_mean}) to ensure the boundedness of the local gradients (line 5), which is typically used in the previous DP guaranteed algorithms. The central server aggregates the clipped local gradients and adds Gaussian noise with mean zero and variance $\sigma_1^2 C_1^2$, where $C_1$ is the clipping radius of the local gradients, to the aggregated gradient and generates an initial differential private global gradient estimator $\widetilde v_1$. In subsequent rounds, the server aggregates the clipped local \emph{gradient differences} $\{\nabla \ell(x_{r-1}, z) - \nabla \ell(x_{r-2}, z)\}_{z \in D_p}$ rather than the local gradients $\{\nabla \ell(x_{r-1}, z)\}_{z \in D_p}$ (line 8). Then, the global gradient estimator $v_r$ is updated as the sum of the aggregated gradient differences and the previous estimator $\widetilde v_{r-1}$ (line 16). Finally, the differential private global gradient estimator $\widetilde v_r$ is obtained by adding Gaussian noise with mean zero and variance $\sigma_2^2 C_{2, r}$, where $C_{2, r}$ is the clipping radius of the local gradient differences. \par

Importantly, for every $T$ rounds, we \emph{restart} these processes; we reset $\widetilde v_{r}$ to the vanilla DP global gradient estimator $\nabla f(x_{r-1}) + \xi_r$ as in the initial round. The reason of the necessary of the restarting will be explained in Section \ref{sec: analysis}.
\par

{\bf{Concrete Algorithm: DIFF2-GD.}} 
Given global gradient estimator $\widetilde v_r$, the simplest choice of general sub-routine is using the gradient descent step based on $\widetilde v_r$. Applying Gradient Descent Routine (GD-Routine, Algorithm \ref{alg: gd_routine}) to DIFF2 gives new algorithm called DIFF2-GD. DIFF2-GD simply executes the single gradient descent step for each round. Note that when $T=1$, DIFF2-GD matches the standard DP-GD.

\section{Differential Privacy Analysis and Utility Analysis}
\label{sec: analysis}
In this section, theoretical analysis of DIFF2-GD is provided. Our analysis is mainly divided into two parts of DP guarantee analysis (Subsection \ref{subsec: dp_analysis_diff2_gd}) and the utility analysis (Subsection \ref{subsec: utility_analysis_diff2_gd}). 

\subsection{DP Guarantee Analysis of DIFF2-GD}
\label{subsec: dp_analysis_diff2_gd}
In this subsection, we investigate the minimum noise levels to guarantee ($\varepsilon_\mathrm{DP}$, $\delta_\mathrm{DP}$)-DP. The proofs are found in Section \ref{app_sec: dp_analysis_diff2_gd} of the supplementary material. 

Thanks to the usage of gradient differences $\{\nabla \ell(x_{r-1}, z) - \nabla \ell(x_{r-2}, z)\}_{z \in D_p}$, one can dramatically reduce the DP noise size compared to using the gradient itself for every round as in DP-GD. This comes from the fact that a gradient difference possibly has much smaller sensitivity than a gradient itself for smooth loss function $\ell$. This is because the $L2$-sensitivity of the gradient difference can be bounded by $L\|x_{r-1} - x_{r-2}\|$, that can be quite small if $x_{r-1}$ and $x_{r-2}$ are sufficiently close, for example if they are around a stationary point, by using the $L$-smoothness of $\ell$: $\|\nabla \ell(x_{r-1}, z) - \nabla \ell(x_{r-2}, z)\| \leq L\|x_{r-1} - x_{r-2}\|$. In contrast, a gradient itself has large $L2$-sensitivity $G$. To derive the DP noise size, our analysis relies on R\'enyi Differential Privacy (RDP) technique \cite{mironov2017renyi} including RDP guarantee of Gaussian mechanism, the composition theorem for multiple RDP mechanisms and the conversion of RDP guarantee to DP one.

The following proposition reveals the minimum noise levels $\sigma_1$ and $\sigma_2$ for ($\varepsilon_\mathrm{DP}$, $\delta_\mathrm{DP}$)-DP guarantee, that can be much smaller than the minimum noise level of DP-GD. 

\begin{proposition}[Minimum Noise Level for $(\varepsilon_\mathrm{DP}, \delta_\mathrm{DP})$-DP]
\label{prop: dp_noise_level_diff2_gd_routine}
Mechanism $\{x_{r}, x_{ r}^\mathrm{out}, \widetilde v_{r}\}_{r \in [R]}$ defined in DIFF2-GD is $(2\alpha \lceil R/T\rceil/(n_p^2P^2\sigma_1^2) + 2\alpha (R -\lceil R/T\rceil) /(n_p^2P^2\sigma_2^2) +  \log(1/\delta_\mathrm{DP})/(\alpha-1), \delta_\mathrm{DP})$-DP for every $\alpha \in (1, \infty)$, where $n_\mathrm{min} := \min\{n_p\}_{p=1}^P$. In particular, setting $\alpha := 1 + \lceil 2\log(1/\delta_\mathrm{DP})/\varepsilon_\mathrm{DP} \rceil$, 
\begin{align*}
    \sigma_1^2 = \frac{4u\alpha \lceil \frac{R}{T}\rceil }{n_\mathrm{min}^2P^2\varepsilon_\mathrm{DP}} = \widetilde \Theta\left(\frac{R}{Tn_\mathrm{min}^2P^2\varepsilon_\mathrm{DP}^2}\right)
\end{align*}
    and 
\begin{align*}
    \sigma_2^2 = \frac{\frac{4u}{u-1}\alpha (R-\lceil \frac{R}{T}\rceil)}{n_\mathrm{min}^2P^2\varepsilon_\mathrm{DP}} = \widetilde \Theta\left( \frac{R}{n_\mathrm{min}^2P^2\varepsilon_\mathrm{DP}^2}\right)
\end{align*} 
is sufficient to guarantee $(\varepsilon_\mathrm{DP}, \delta_\mathrm{DP})$ of mechanism $\{x_{r}, x_{r}^\mathrm{out}, \widetilde v_{r}\}_{r \in [R]}$ for any $u > 1$.
\end{proposition}

\begin{remark}
Compared with the noise size $\sigma^2 C_1^2 = \widetilde \Theta(C_1^2 R/(n_\mathrm{min} P\varepsilon_\mathrm{DP})^2)$ of DP-GD, $\sigma_1^2 C_1^2$ is $T$ times smaller and $\sigma_2^2 C_{2, r}^2$ is possibly much smaller since $C_{2, r} = C_2\|x_{r-1} - x_{r-2}\| \ll C_1$ may hold. 
\end{remark}
\subsection{Utility Analysis of DIFF2-GD}
\label{subsec: utility_analysis_diff2_gd}
In this subsection, utility analysis of DIFF2-GD is provided. The proofs are found in Section \ref{app_sec: utility_analysis_diff2_gd} of the supplementary material.

As explained in Section \ref{sec: analysis}, it is possible to reduce the DP noise size and thus the variance of the global gradient estimator $\widetilde v_r$ in DIFF2 framework, and better utility is expected compared to the standard DP-GD. On the other side, we also need to consider the \emph{bias} of $\widetilde v_r$; one can observe that $\widetilde v_r$ is not anymore an unbiased estimator of $\nabla f(x_{r-1})$ because $\widetilde v_r$ contains the previous history of the DP noise. Generally, the larger the restart interval $T$, the smaller the variance of $\widetilde v_r$, but the larger the bias of $\widetilde v_r$.  Hence, the main challenge of our utility analysis is to carefully evaluate both the variance and the bias of $\widetilde v_r$ and determine the optimal choice of $T$ that controls the trade-off relationship between the two errors. 

The following theorem states that with the optimal choice of $T$, DIFF2 achieves better utility than the best known one of DP-GD while ($\varepsilon_\mathrm{DP}$, $\delta_\mathrm{DP}$)-DP is guaranteed.

\begin{theorem}[Utility Bound]
\label{theorem: utility_DIFF2_gd_routine}
Suppose that Assumptions \ref{assump: local_loss_smoothness}, \ref{assump: optimal_sol} and \ref{assump: local_loss_gradient_boundedness} hold. Assume that $C_1 \geq G$, $C_1 = \Theta(G)$, $C_2 \geq L$, $C_2 = \Theta(L)$, $f(x_0) - f(x_*) = O(1)$ and  $n_\mathrm{min}P = \Omega(G^2\sqrt{d}/(L\varepsilon_\mathrm{DP}))$. Under the choices of $\sigma_1^2$ and $\sigma_2^2$ in Proposition \ref{prop: dp_noise_level_diff2_gd_routine}, if we appropriately choose $\eta = \Theta(\min\{1/L, 1/(\sqrt{T}L\sigma_2\sqrt{d}\})$ and $T = \Theta(\max\{1,\tau R\})$ with $\tau := (G^2\sqrt{d}/(Ln_\mathrm{min}P\varepsilon_\mathrm{DP}))^{2/3}$, DIFF2-GD satisfies
\begin{align*}
    \mathbb{E}\|\nabla f(x^\mathrm{out})\|^2 \leq&\  O\left( \frac{L}{R}\right) + \widetilde O\left(\frac{L\sqrt{d}}{n_\mathrm{min}P\varepsilon_\mathrm{DP}\sqrt{R}} + \varepsilon_\mathrm{opt}\right),  
\end{align*}
where $\varepsilon_\mathrm{opt} := \widetilde \Theta((LGd)^\frac{2}{3}/(n_\mathrm{min}P\varepsilon_\mathrm{DP})^\frac{4}{3})$.
In particular, setting 
\begin{align*}
    R = \Theta\left( 1 + \frac{L}{\varepsilon_\mathrm{opt}}\right) + \widetilde \Theta\left( \frac{L^2d}{n_\mathrm{min}^2P^2\varepsilon_\mathrm{DP}^2\varepsilon_\mathrm{opt}^2}\right)
\end{align*}
results in utility $\mathbb{E}\|\nabla f(x^\mathrm{out})\|^2 \leq \varepsilon_\mathrm{opt}$.
\end{theorem}

\begin{remark}
The obtained utility is significantly better than the best known utility bound $\widetilde O(\sqrt{L}G\sqrt{d}/(n_\mathrm{min}P\varepsilon_\mathrm{DP}))$ when $n_\mathrm{min}P \gg \sqrt{Ld}/(G\varepsilon_\mathrm{DP})$. Besides, when $n_\mathrm{min}P = \widetilde \Omega((Ld)^2/(G\varepsilon_\mathrm{opt}))$, the necessary communication rounds $R$ becomes $O(1 + L/\varepsilon_\mathrm{opt})$ and the necessary number of per sample gradient evaluations for client $p$ becomes $ O(n_p(1+L/\varepsilon_\mathrm{opt}))$.   
\end{remark}

\section{More Efficient Sub-Routine}
\label{sec: efficient_sub_routine}
In this section, we briefly discuss a more  efficient sub-routine for DIFF2 framework in terms of both computation and communication cost. This point is not our main focus, but important for practical use of DIFF2 framework. Specifically, we propose a sub-routine involving local optimization based on BVR-L-SGD algorithm \cite{murata2021bias}. The sub-routine enjoys nice computation and communication efficiency by inheriting the bias-variance reduction property of BVR-L-SGD. The objective of this section is to derive a sub-routine for DIFF2 that requires lower computational and communication cost than DIFF2-GD, while it maintains the same utility as DIFF2-GD. 

\paragraph{DIFF2-BVR-L-SGD}
One problem with DIFF2-GD is that it does not utilize stochastic gradients and local optimization for each round. This may result in both high computational cost and high communication cost. Thus, it is natural to apply local optimization with stochastic local gradients to reduce both the computation and communication cost. BVR-L-SGD-Routine (Algorithm \ref{alg: BVR-L-SGD_routine}) inspired by BVR-L-SGD \citep{murata2021bias} is a nice alternative to GD-Routine to solve the aforementioned two problems. In each local iteration $k \in [K]$, BVR-L-SGD-Routine updates the bias-variance reduced stochastic gradient estimator $v_{k, r}$ based on the clipped mean of the per sample gradient differences similar to DIFF2 (line 8). Then, Gaussian noise $\xi_{k, r}$ with mean zero and variance $\sigma_3^2C_3^2\|x_{k-1, r-1} - x_{k-2, r-1}\|^2$ is added to $v_{k, r}$ and the current parameter is updated based on the differentially private gradient estimator $\widetilde v_{k, r} := v_{k, r} + \xi_{k, r}$ (line 9). 
We call BVR-L-SGD-Routine combined with DIFF2 as DIFF2-BVR-L-SGD. 
Note that when $T=1$, DIFF2-BVR-L-SGD without DP noise matches the special case of BVR-L-SGD. 

\begin{algorithm}[t]
\caption{BVR-L-SGD-Routine($x_{0, r-1}$, $\widetilde v_{1, r}$, $\eta$, $b$, $\sigma_3$, $C_3$, $K$)}
\label{alg: BVR-L-SGD_routine}
\begin{algorithmic}[1]
\STATE Set $x_{1, r-1} = x_{0, r-1} - \eta \widetilde v_{1, r}$.
\STATE Set $p_r = 1 + r \% P$. 
\STATE Client $p_r$ receives $x_{1, r-1}$ from the central server. 
\STATE \# Only client $p_r$ runs local optimization.
\FOR {$k=2$ to $K$}
    \STATE
    Draw random subset $I_{k, r}$ with size $b$ from $D_{p_r}$ without replacement. 
    \STATE Set $C_{3, k, r}= C_3\|x_{k-1, r-1} - x_{k-2, r-1}\|$.
    \STATE 
    $v_{k, r} = $ ClippedMean$(\{\nabla \ell(x_{k-1, r-1}, z) - \nabla \ell(x_{k-2, r-1}, z)\}_{z \in I_{k, r}}, C_{3, k, r}) + \widetilde v_{k-1, r}$.  
    \STATE Set $\widetilde v_{k, r} = v_{k, r} + \xi_{k, r}$, where $\xi_{k, r} \sim N(0, \sigma_3^2 C_{3, k, r}^2I)$.
    \STATE Update $x_{k, r-1} = x_{k-1, r-1} - \eta \widetilde v_{k, r}$.
\ENDFOR
\STATE Client $\hat p$ sends $x_{K, r-1}$ and $x_{\hat k-1, r-1}$ to the central server, where $\hat k \sim \mathrm{Unif}[K]$.
\STATE {\bf{Return:}} $x_{K, r-1}$, $x_{\hat k, r-1}$.
\end{algorithmic}
\end{algorithm}

Similar to the analysis of DIFF2-GD, we can derive the minimum noise level of DIFF2-BVR-L-SGD for $(\varepsilon_\mathrm{DP}, \delta_\mathrm{DP})$-DP guarantee. Importantly, we need to make use of the subsampling amplification technique of RDP developed in \citet{wang2019subsampled} for tight analysis due to the usage of stochastic gradients. For the detailed statements, see Proposition \ref{app_prop: dp_noise_level_DIFF2_BVR-L-SGD_routine} in the supplementary material. 
\begin{proposition}[Minimum Noise Level for $(\varepsilon_\mathrm{DP}, \delta_\mathrm{DP})$-DP (Simplified)]
\label{prop: dp_noise_level_DIFF2_BVR-L-SGD_routine}
There exist $\sigma_1^2$, $\sigma_2^2$ and $\sigma_3^2$ that guarantee $(\varepsilon_\mathrm{DP}, \delta_\mathrm{DP})$-DP of DIFF2-BVR-L-SGD such that $\sigma_1^2 = \widetilde \Theta((R/(Tn_\mathrm{min}^2 P^2\varepsilon_\mathrm{DP}^2))$, $\sigma_2^2 = \widetilde \Theta(R/(n_\mathrm{min}^2 P^2\varepsilon_\mathrm{DP}^2))$ and $\sigma_3^2 =  \widetilde \Theta((KR/(n_\mathrm{min}^2P\varepsilon_\mathrm{DP}^2)) \vee (1/(b^2\varepsilon_\mathrm{DP})))$ under $b \leq \min\{n_\mathrm{min}/(2e\alpha),  (4n_\mathrm{min}/(\alpha\sigma_3^2))^{1/3}\}$.
\end{proposition}

To show the communication efficiency of DIFF2-BVR-L-SGD, the following assumption is additionally used. This assumption is often used in non-DP distributed optimization \cite{karimireddy2020mime, murata2021bias}. Note that the assumption always holds with $\zeta := L$ under Assumption \ref{assump: local_loss_smoothness} and we implicitly expect that $\zeta \ll L$. 
 
\begin{assumption}[Hessian heterogeneity]
\label{assump: similarity}
$\{f_p\}_{p=1}^P$ is $\zeta$-Hessian heterogeneity, that is  
\begin{align*}
    \left\|\nabla^2 f_p(x) - \nabla^2 f(x)\right\| \leq \zeta, \forall x \in \mathbb{R}^d.
\end{align*}
\end{assumption}

To derive a utility bound of DIFF2-BVR-L-SGD, we need to bound the bias and variance caused by DP noise and \emph{additionally local optimization and stochastic gradients noise}. The following theorem shows that DIFF2-BVR-L-SGD achieves the same utility as DIFF2-GD. For the detailed statements, see Theorem \ref{app_theorem: utility} in the supplementary material. 
\begin{theorem}[Utility Bound (Simplified)]
\label{theorem: utility}
Suppose that the assumptions of Theorem \ref{theorem: utility_DIFF2_gd_routine} 
hold. Additionally, suppose that Assumption \ref{assump: similarity} hold and minibatch size $b$ satisfies the condition in Proposition \ref{prop: dp_noise_level_DIFF2_BVR-L-SGD_routine}. Let $C_3 := C_2$. Under the choices of $\sigma_1^2$, $\sigma_2^2$ and $\sigma_3^2$ in Proposition \ref{prop: dp_noise_level_DIFF2_BVR-L-SGD_routine}, if appropriate $\eta$ and $T$ are chosen, DIFF2-BVR-L-SGD achieves utility of $$\mathbb{E}\|\nabla f(x^\mathrm{out})\|^2 \leq \varepsilon_\mathrm{opt} := \widetilde \Theta\left(\frac{(LGd)^\frac{2}{3}}{(n_\mathrm{min}P\varepsilon_\mathrm{DP})^\frac{4}{3}}\right)$$
by setting $R = \Theta(1) + c_1/\varepsilon_\mathrm{opt} + c_2/\varepsilon_\mathrm{opt}^2$, where $c_1 :=\Theta(L/K + \zeta + L/(\sqrt{K}b)) + \widetilde \Theta( L\sqrt{d}/(\sqrt{K}b\sqrt{\varepsilon_\mathrm{DP}}))$ and $c_2 := \widetilde \Theta(L^2d/(n_\mathrm{min}^2P\varepsilon_\mathrm{DP}^2))$.
\end{theorem}
\begin{remark}
[Communication efficiency] DIFF2-BVR-L-SGD can be much communication efficient than DIFF2-GD. 
Actually, when $n_\mathrm{min}P = \widetilde \Omega(L^2\sqrt{Pd}/(G\varepsilon_\mathrm{DP}))$, it holds that $c_2/\varepsilon_\mathrm{opt}^2 \leq O(1)$ and thus the necessary communication rounds of DIFF2-BVR-L-SGD becomes $O(1 + c_1/\varepsilon_\mathrm{opt})$, that is much smaller than $O(1 + L/\varepsilon_\mathrm{opt})$ of DIFF2-GD when $K \gg 1$ and $\zeta \ll L$. 
\end{remark}

\begin{remark}[Computation efficiency]
DIFF2-BVR-L-SGD is possibly more computation efficient than DIFF2-GD. Actually, the total number of single gradient evaluations of client $p$ to achieve the best possible utility is $O((n_p + Kb)R)$. As illustrated in the remark just above, the number of communication rounds $R$ can be much smaller than the one of DP-GD and thus the total computational cost can be also smaller than the one of DP-GD when $Kb = O(n_p)$. 
\end{remark}

\section{Numerical Results}
In this section, we provide some experimental results to verify our theoretical findings in Section \ref{sec: analysis}. Specifically, we empirically compare the utility of DIFF2-GD with the one of DP-GD and validate the superiority of DIFF2 framework. We focused on relatively low dimensional problems, where $n_\mathrm{min}P \gg d$ and thus the utility $\widetilde O(d^{2/3}/(n_\mathrm{min}P\varepsilon_\mathrm{DP})^{4/3})$ of DIFF2-GD was expected to be better than $\widetilde O(\sqrt{d}/(n_\mathrm{min}P\varepsilon_\mathrm{DP}))$ of DP-GD.  
\par
We conducted regression and classification tasks on five dataset: (i) California Housing Data Set \footnotemark\footnotetext{\url{https://www.dcc.fc.up.pt/~ltorgo/Regression/cal_housing.html}.}; (ii) Gas Turbine CO and NOx Emission Data Set \footnotemark\footnotetext{\url{https://archive.ics.uci.edu/ml/datasets/Gas+Turbine+CO+and+NOx+Emission+Data+Set}. We removed attribute NOx and used CO as the target variable.}; (iii) BlogFeedback Data Set \footnotemark\footnotetext{\url{https://archive.ics.uci.edu/ml/datasets/BlogFeedback}.}; (iv) KDDCup99 Data Set \footnotemark\footnotetext{\url{https://kdd.ics.uci.edu/databases/kddcup99/kddcup99.html}}; and (v) Cover Type Dataset \footnotemark\footnotetext{\url{http://archive.ics.uci.edu/ml/datasets/covertype}}. We only report the results of regression tasks in the main paper. A summary of the datasets and the results of classification tasks are provided in Section \ref{app_sec: experiment} of the supplementary material. 
\par

{\bf{Data Preparation. }}
For each dataset, we randomly split the orginal dataset into a $80$ \% train dataset and a $20$ \% test dataset. Then, we normalized the numeric  attributes to have mean zero and standard deviation one. Also, we normalized the target variable by dividing the maximum absolute value of the target variable in the whole dataset. We randomly split the train dataset into equal-sized $10$ subsets and each subset was assigned to the corresponding worker of $10$ workers\footnotemark.
\footnotetext{This means that the local datasets were distributed in I.I.D. manner. Since we focused on the comparison of DIFF2-GD with DP-GD  and these methods do not rely on local optimization, the heterogeneity never affects the optimization processes. Hence, we decided to use the simplest I.I.D. case. }
\par
{\bf{Models. }}We conducted our experiments using a one-hidden layer fully connected neural network with $10$ hidden units and softplus activation. For loss function, we used the squared loss. We initialized parameters by uniformly sampling the parameters from $[-\sqrt{w_\mathrm{in}}, \sqrt{w_\mathrm{in}}]$, where $w_{\mathrm{in}}$ is the number of units in the input layers. 
\par
{\bf{Implemented Algorithms. }} Non-private GD, DP-GD and DIFF2-GD were implemented. For DP-GD and DIFF2-GD, the differential privacy parameters were set to $\varepsilon_\mathrm{DP} \in \{3.0, 5.0\}$ and $\delta_\mathrm{DP} = 10^{-5}$. We used Proposition \ref{prop: dp_noise_level_diff2_gd_routine} to determine the DP noise size of DP-GD ($T=1$, $u \to 1$) and DIFF2-GD ($T \in \{0.003R, 0.01R, 0.03R, 0.1R\}$, $u = 1.25$). For each algorithm, we automatically tuned the hyper-parameters. The details of the tuning procedures are found in Section \ref{app_sec: experiment} of the supplementary material. 
\par

{\bf{Evaluation. }}
We evaluated the implemented algorithms using three criteria of train loss; squared train gradient norm; and test loss against the number of communication rounds. We report the train loss and the squared train gradient norm of the implemented algorithms with the hyper-parameters that minimized each criteria. On the other side, we report the test loss of the implemented algorithms with the hyper-parameters that minimized the train loss.  
The total number of communication rounds was fixed to $2,000$ for each algorithm. We independently repeated the experiments $5$ times and report the mean and the standard deviation of the above criteria. 
\par
\begin{figure*}[t]
\begin{subfigmatrix}{3}
\subfigure[Train Loss]{\includegraphics[width=5.5cm]{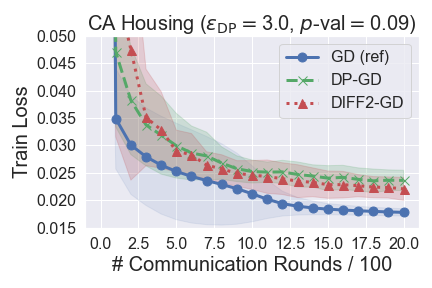}
}
\subfigure[Squared Train Grad. Norm]{\includegraphics[width=5.5cm]{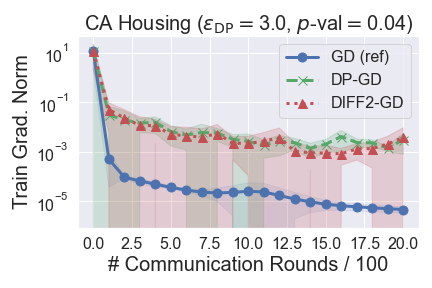}
}
\subfigure[Test Loss]
{\includegraphics[width=5.5cm]{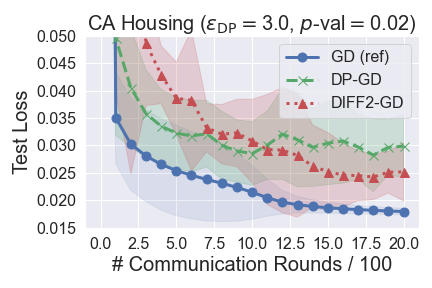}
}
\end{subfigmatrix}
\vspace{-0.5em}
\begin{subfigmatrix}{3}
\subfigure[Train Loss]{\includegraphics[width=5.5cm]{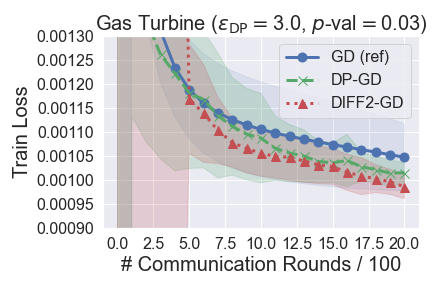}
}
\subfigure[Squared Train Grad. Norm]{\includegraphics[width=5.5cm]{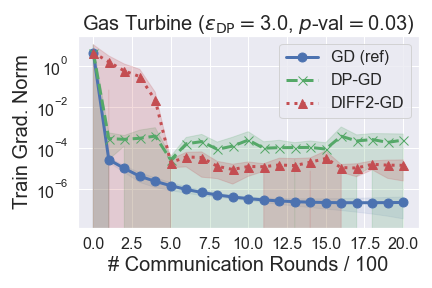}
}
\subfigure[Test Loss]
{\includegraphics[width=5.5cm]{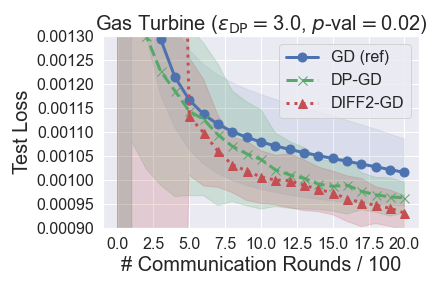}
}
\end{subfigmatrix}
\vspace{-0.5em}
\begin{subfigmatrix}{3}
\subfigure[Train Loss]{\includegraphics[width=5.5cm]{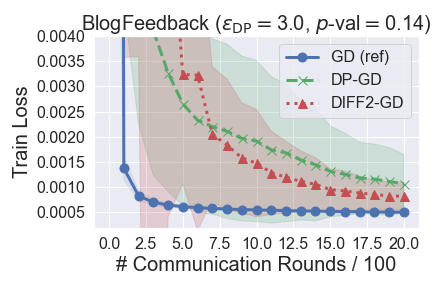}
}
\subfigure[Squared Train Grad. Norm]{\includegraphics[width=5.5cm]{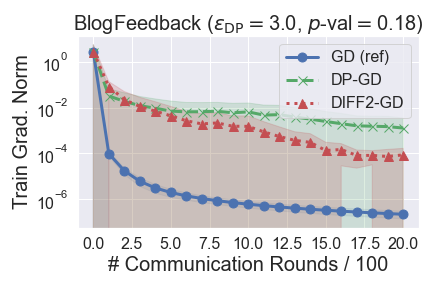}
}
\subfigure[Test Loss]
{\includegraphics[width=5.5cm]{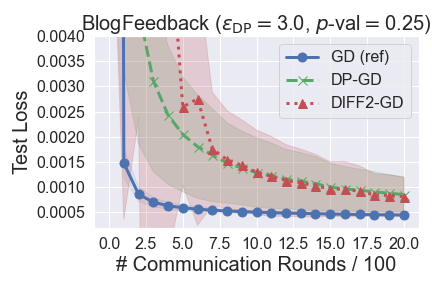}
}
\vspace{-0.5em}
\end{subfigmatrix}
\caption{Comparison of the train loss, train gradient norm and test loss against the number of communication rounds ($\varepsilon=3.0$, $\delta=10^{-5}$ and $R = 2,000$). (a)-(c) shows the comparison of the three criteria on California Housing dataset, (d)-(f) do the ones on Gas Turbine CO and NOx Emission dataset and (g)-(i) do the ones on BlogFeedback dataset. DIFF2-GD consistently outperformed DP-GD. }
\label{fig: by_rounds}
\end{figure*}

{\bf{Results. }} Figure \ref{fig: by_rounds} shows the comparison of DIFF2-GD with DP-GD on the three datasets for $\varepsilon_\mathrm{DP} = 3.0$. ``GD (ref)" means gradient descent without DP noise. ``$p$-value" in the title of each figure means the $p$-value of the one-sided $t$-test with the alternative hypothesis that the difference of the minimum value in $2,000$ rounds of DIFF2-GD from the one of DP-GD is negative, where one random seed corresponded to one sample, i.e., the degree of freedom of $t$-distribution was four. We can observe that DIFF2-GD consistently outperformed DP-GD at the final learning rounds for each metric, although it sometimes showed slower convergence than DP-GD at initial learning rounds. Also, the results of $t$-test supported statistical significance of the superiority of DIFF2-GD. The case of $\varepsilon_\mathrm{DP} = 5.0$ showed similar results to the case of $\varepsilon_\mathrm{DP} = 3.0$, that are provided in the supplementary material due to the space limitation.  
 
\section{Conclusion and Future Work}
In this work, we proposed a new differential private optimization framework called DIFF2 
to improve the previous best known utility bound. DIFF2 communicates gradient differences rather than gradients themselves and constructs a differential private global gradient estimator with possibly quite small variance based on the gradient differences. It is shown that DIFF2 with a gradient descent subroutine achieves the utility of $\widetilde O(d^{2/3}/(n_\mathrm{min}P\varepsilon_\mathrm{DP})^{4/3})$, which can be significantly better than the previous one in terms of the dependence on sample size $n_\mathrm{min}P$. We conducted numerical experiments and confirmed the superiority of DIFF2-GD to DP-GD and the benefit of DIFF2 framework. 

The first important future work is to derive a lower bound of the utility of $(\varepsilon_\mathrm{DP}, \delta_\mathrm{DP})$-DP optimization algorithms for nonconvex smooth objectives and discuss the optimality of DIFF2 framework. The second important future work is to develop new utility analysis of DP optimization algorithms for high dimensional problems, which often arise in learning modern huge deep neural networks. While this work has tackled a quite fundamental and important problem, utility improvement, in DP optimization theory, our theoretical improvement is limited on relatively low dimensional problems. Finally, extensions of our algorithms to various directions are also important future work. For example, client sampling is typically used in cross-device FL, but is not discussed in this work. Also, DIFF2 framework on a decentralized topology is an important extension. It is non-trivial to extend our DP guarantees and utility analysis to these settings. While a more computational and communication efficient sub-routine than GD is briefly mentioned in Section \ref{sec: efficient_sub_routine}, a broad range of optimization algorithms including simple SGD, accelerated gradient methods, adaptive gradient methods and so on can be incorporated into DIFF2 framework. Then, it is important to develop unified theoretical analysis and compare their theoretical and empirical performances.

\section*{Acknowledgement}
TS was partially supported by JSPS KAKENHI (20H00576) and JST CREST. 

\bibliography{references}

\begin{thebibliography}{62}
\providecommand{\natexlab}[1]{#1}
\providecommand{\url}[1]{\texttt{#1}}
\expandafter\ifx\csname urlstyle\endcsname\relax
  \providecommand{\doi}[1]{doi: #1}\else
  \providecommand{\doi}{doi: \begingroup \urlstyle{rm}\Url}\fi

\bibitem[Abadi et~al.(2016)Abadi, Chu, Goodfellow, McMahan, Mironov, Talwar,
  and Zhang]{abadi2016deep}
Abadi, M., Chu, A., Goodfellow, I., McMahan, H.~B., Mironov, I., Talwar, K.,
  and Zhang, L.
\newblock Deep learning with differential privacy.
\newblock In \emph{ACM SIGSAC Conference on Computer and Communications
  Security}, pp.\  308--318, 2016.

\bibitem[Agarwal et~al.(2018)Agarwal, Suresh, Yu, Kumar, and
  McMahan]{agarwal2018cpsgd}
Agarwal, N., Suresh, A.~T., Yu, F. X.~X., Kumar, S., and McMahan, B.
\newblock cp{SGD}: Communication-efficient and differentially-private
  distributed {SGD}.
\newblock In \emph{Advances in Neural Information Processing Systems},
  volume~31, 2018.

\bibitem[Asi et~al.(2021)Asi, Feldman, Koren, and Talwar]{asi2021private}
Asi, H., Feldman, V., Koren, T., and Talwar, K.
\newblock Private stochastic convex optimization: Optimal rates in $l_1$
  geometry.
\newblock In \emph{International Conference on Machine Learning}, volume 139,
  pp.\  393--403, 2021.

\bibitem[Bassily et~al.(2014)Bassily, Smith, and Thakurta]{bassily2014private}
Bassily, R., Smith, A., and Thakurta, A.
\newblock Private empirical risk minimization: Efficient algorithms and tight
  error bounds.
\newblock In \emph{IEEE Symposium on Foundations of Computer Science},
  volume~55, pp.\  464--473, 2014.

\bibitem[Bassily et~al.(2019)Bassily, Feldman, Talwar, and
  Guha~Thakurta]{bassily2019private}
Bassily, R., Feldman, V., Talwar, K., and Guha~Thakurta, A.
\newblock Private stochastic convex optimization with optimal rates.
\newblock In \emph{Advances in Neural Information Processing Systems},
  volume~32, 2019.

\bibitem[Chaudhuri et~al.(2011)Chaudhuri, Monteleoni, and
  Sarwate]{chaudhuri2011differentially}
Chaudhuri, K., Monteleoni, C., and Sarwate, A.~D.
\newblock Differentially private empirical risk minimization.
\newblock \emph{Journal of Machine Learning Research}, 12, 2011.

\bibitem[Chourasia et~al.(2021)Chourasia, Ye, and
  Shokri]{chourasia2021differential}
Chourasia, R., Ye, J., and Shokri, R.
\newblock Differential privacy dynamics of {L}angevin diffusion and noisy
  gradient descent.
\newblock In \emph{Advances in Neural Information Processing Systems},
  volume~34, pp.\  14771--14781, 2021.

\bibitem[Cutkosky \& Orabona(2019)Cutkosky and Orabona]{cutkosky2019momentum}
Cutkosky, A. and Orabona, F.
\newblock Momentum-based variance reduction in non-convex sgd.
\newblock In \emph{Advances in neural information processing systems},
  volume~32, 2019.

\bibitem[Das et~al.(2022)Das, Kale, Xu, Zhang, and Sanghavi]{das2022beyond}
Das, R., Kale, S., Xu, Z., Zhang, T., and Sanghavi, S.
\newblock Beyond uniform {L}ipschitz condition in differentially private
  optimization.
\newblock \emph{arXiv preprint arXiv:2206.10713}, 2022.

\bibitem[Defazio et~al.(2014)Defazio, Bach, and
  Lacoste-Julien]{defazio2014saga}
Defazio, A., Bach, F., and Lacoste-Julien, S.
\newblock {SAGA}: A fast incremental gradient method with support for
  non-strongly convex composite objectives.
\newblock In \emph{Advances in Neural Information Processing Systems},
  volume~27, pp.\  1646--1654, 2014.

\bibitem[Ding et~al.(2021)Ding, Liang, Bi, and Pan]{ding2021differentially}
Ding, J., Liang, G., Bi, J., and Pan, M.
\newblock Differentially private and communication efficient collaborative
  learning.
\newblock In \emph{AAAI Conference on Artificial Intelligence}, volume~35, pp.\
   7219--7227, 2021.

\bibitem[Du et~al.(2021)Du, Li, Chen, Chen, and Hong]{du2021dynamic}
Du, J., Li, S., Chen, X., Chen, S., and Hong, M.
\newblock Dynamic differential-privacy preserving {SGD}.
\newblock \emph{arXiv preprint arXiv:2111.00173}, 2021.

\bibitem[Dwork(2011)]{dwork2011firm}
Dwork, C.
\newblock A firm foundation for private data analysis.
\newblock \emph{Communications of the ACM}, 54:\penalty0 86--95, 2011.

\bibitem[Dwork \& Naor(2010)Dwork and Naor]{dwork2010difficulties}
Dwork, C. and Naor, M.
\newblock On the difficulties of disclosure prevention in statistical databases
  or the case for differential privacy.
\newblock \emph{Journal of Privacy and Confidentiality}, 2:\penalty0 93--107,
  2010.

\bibitem[Dwork et~al.(2006)Dwork, McSherry, Nissim, and
  Smith]{dwork2006calibrating}
Dwork, C., McSherry, F., Nissim, K., and Smith, A.
\newblock Calibrating noise to sensitivity in private data analysis.
\newblock In \emph{Theory of Cryptography Conference}, pp.\  265--284, 2006.

\bibitem[Dwork et~al.(2010)Dwork, Rothblum, and Vadhan]{dwork2010boosting}
Dwork, C., Rothblum, G.~N., and Vadhan, S.
\newblock Boosting and differential privacy.
\newblock In \emph{IEEE Annual Symposium on Foundations of Computer Science},
  pp.\  51--60, 2010.

\bibitem[Dwork et~al.(2014)Dwork, Roth, et~al.]{dwork2014algorithmic}
Dwork, C., Roth, A., et~al.
\newblock The algorithmic foundations of differential privacy.
\newblock \emph{Foundations and Trends{\textregistered} in Theoretical Computer
  Science}, 9:\penalty0 211--407, 2014.

\bibitem[Fredrikson et~al.(2015)Fredrikson, Jha, and
  Ristenpart]{fredrikson2015model}
Fredrikson, M., Jha, S., and Ristenpart, T.
\newblock Model inversion attacks that exploit confidence information and basic
  countermeasures.
\newblock In \emph{ACM SIGSAC Conference on Computer and Communications
  Security}, pp.\  1322--1333, 2015.

\bibitem[Geiping et~al.(2020)Geiping, Bauermeister, Dr{\"o}ge, and
  Moeller]{geiping2020inverting}
Geiping, J., Bauermeister, H., Dr{\"o}ge, H., and Moeller, M.
\newblock Inverting gradients-how easy is it to break privacy in federated
  learning?
\newblock In \emph{Advances in Neural Information Processing Systems},
  volume~33, pp.\  16937--16947, 2020.

\bibitem[Geyer et~al.(2017)Geyer, Klein, and Nabi]{geyer2017differentially}
Geyer, R.~C., Klein, T., and Nabi, M.
\newblock Differentially private federated learning: A client level
  perspective.
\newblock \emph{arXiv preprint arXiv:1712.07557}, 2017.

\bibitem[Han et~al.(2022)Han, Mishra, Jawanpuria, and
  Gao]{han2022differentially}
Han, A., Mishra, B., Jawanpuria, P., and Gao, J.
\newblock Differentially private {R}iemannian optimization.
\newblock \emph{arXiv preprint arXiv:2205.09494}, 2022.

\bibitem[Hu et~al.(2022)Hu, Ni, Xiao, and Wang]{hu2022high}
Hu, L., Ni, S., Xiao, H., and Wang, D.
\newblock High dimensional differentially private stochastic optimization with
  heavy-tailed data.
\newblock In \emph{ACM SIGMOD-SIGACT-SIGAI Symposium on Principles of Database
  Systems}, pp.\  227--236, 2022.

\bibitem[Huang \& Gong(2020)Huang and Gong]{huang2020differentially}
Huang, Z. and Gong, Y.
\newblock Differentially private {ADMM} for convex distributed learning:
  Improved accuracy via multi-step approximation.
\newblock \emph{arXiv preprint arXiv:2005.07890}, 2020.

\bibitem[Huang et~al.(2019)Huang, Hu, Guo, Chan-Tin, and Gong]{huang2019dp}
Huang, Z., Hu, R., Guo, Y., Chan-Tin, E., and Gong, Y.
\newblock {DP-ADMM}: {ADMM}-based distributed learning with differential
  privacy.
\newblock \emph{IEEE Transactions on Information Forensics and Security},
  15:\penalty0 1002--1012, 2019.

\bibitem[Johnson \& Zhang(2013)Johnson and Zhang]{johnson2013accelerating}
Johnson, R. and Zhang, T.
\newblock Accelerating stochastic gradient descent using predictive variance
  reduction.
\newblock In \emph{Advances in Neural Information Processing Systems},
  volume~26, pp.\  315--323, 2013.

\bibitem[Kamath et~al.(2022)Kamath, Liu, and Zhang]{kamath2022improved}
Kamath, G., Liu, X., and Zhang, H.
\newblock Improved rates for differentially private stochastic convex
  optimization with heavy-tailed data.
\newblock In \emph{International Conference on Machine Learning}, volume 162,
  pp.\  10633--10660, 2022.

\bibitem[Karimireddy et~al.(2021)Karimireddy, Jaggi, Kale, Mohri, Reddi, Stich,
  and Suresh]{karimireddy2020mime}
Karimireddy, S.~P., Jaggi, M., Kale, S., Mohri, M., Reddi, S., Stich, S.~U.,
  and Suresh, A.~T.
\newblock Breaking the centralized barrier for cross-device federated learning.
\newblock In \emph{Advances in Neural Information Processing Systems},
  volume~34, pp.\  28663--28676, 2021.

\bibitem[Kone{\v{c}}n{\`y} et~al.(2016)Kone{\v{c}}n{\`y}, McMahan, Yu,
  Richt{\'a}rik, Suresh, and Bacon]{konevcny2016federated}
Kone{\v{c}}n{\`y}, J., McMahan, H.~B., Yu, F.~X., Richt{\'a}rik, P., Suresh,
  A.~T., and Bacon, D.
\newblock Federated learning: Strategies for improving communication
  efficiency.
\newblock \emph{arXiv preprint arXiv:1610.05492}, 2016.

\bibitem[Kuru et~al.(2022)Kuru, Birbil, Gurbuzbalaban, and
  Yildirim]{kuru2022differentially}
Kuru, N., Birbil, I., Gurbuzbalaban, M., and Yildirim, S.
\newblock Differentially private accelerated optimization algorithms.
\newblock \emph{SIAM Journal on Optimization}, 32:\penalty0 795--821, 2022.

\bibitem[Li et~al.(2006)Li, Li, and Venkatasubramanian]{li2006t}
Li, N., Li, T., and Venkatasubramanian, S.
\newblock $t$-closeness: Privacy beyond $k$-anonymity and $l$-diversity.
\newblock In \emph{IEEE International Conference on Data Engineering}, pp.\
  106--115, 2006.

\bibitem[Li et~al.(2022)Li, Zhao, Li, and Chi]{li2022soteriafl}
Li, Z., Zhao, H., Li, B., and Chi, Y.
\newblock Soteria{FL}: A unified framework for private federated learning with
  communication compression.
\newblock In \emph{Advances in Neural Information Processing Systems}, 2022.
\newblock to appear.

\bibitem[Liang et~al.(2020)Liang, Wang, Gu, Osher, and
  Yao]{liang2020differentially}
Liang, Z., Wang, B., Gu, Q., Osher, S., and Yao, Y.
\newblock Differentially private federated learning with {L}aplacian smoothing.
\newblock \emph{arXiv preprint arXiv:2005.00218}, 2020.

\bibitem[Liu \& Lu(2021)Liu and Lu]{liu2021tight}
Liu, D. and Lu, Z.
\newblock Tight lower bounds for differentially private {ERM}.
\newblock \emph{OpenReview preprint
  https://openreview.net/forum?id=30nbp1eV0dJ}, 2021.

\bibitem[Machanavajjhala et~al.(2007)Machanavajjhala, Kifer, Gehrke, and
  Venkitasubramaniam]{machanavajjhala2007diversity}
Machanavajjhala, A., Kifer, D., Gehrke, J., and Venkitasubramaniam, M.
\newblock $l$-diversity: Privacy beyond $k$-anonymity.
\newblock \emph{ACM Transactions on Knowledge Discovery from Data}, 1:\penalty0
  3--es, 2007.

\bibitem[McMahan et~al.(2017)McMahan, Moore, Ramage, Hampson, and
  y~Arcas]{mcmahan2017communication}
McMahan, B., Moore, E., Ramage, D., Hampson, S., and y~Arcas, B.~A.
\newblock Communication-efficient learning of deep networks from decentralized
  data.
\newblock In \emph{International Conference on Artificial Intelligence and
  Statistics}, volume~54, pp.\  1273--1282, 2017.

\bibitem[McMahan et~al.(2018)McMahan, Ramage, Talwar, and
  Zhang]{mcmahan2017learning}
McMahan, H.~B., Ramage, D., Talwar, K., and Zhang, L.
\newblock Learning differentially private recurrent language models.
\newblock In \emph{International Conference on Learning Representations},
  volume~6, 2018.

\bibitem[Mironov(2017)]{mironov2017renyi}
Mironov, I.
\newblock R{\'e}nyi differential privacy.
\newblock In \emph{IEEE Computer Security Foundations Symposium}, volume~30,
  pp.\  263--275, 2017.

\bibitem[Murata \& Suzuki(2021)Murata and Suzuki]{murata2021bias}
Murata, T. and Suzuki, T.
\newblock Bias-variance reduced local {SGD} for less heterogeneous federated
  learning.
\newblock In \emph{International Conference on Machine Learning}, volume 139,
  pp.\  7872--7881, 2021.

\bibitem[Nasr et~al.(2019)Nasr, Shokri, and Houmansadr]{nasr2019comprehensive}
Nasr, M., Shokri, R., and Houmansadr, A.
\newblock Comprehensive privacy analysis of deep learning: Passive and active
  white-box inference attacks against centralized and federated learning.
\newblock In \emph{IEEE Symposium on Security and Privacy}, pp.\  739--753,
  2019.

\bibitem[Nguyen et~al.(2017)Nguyen, Liu, Scheinberg, and
  Tak{\'a}{\v{c}}]{nguyen2017sarah}
Nguyen, L.~M., Liu, J., Scheinberg, K., and Tak{\'a}{\v{c}}, M.
\newblock {SARAH}: A novel method for machine learning problems using
  stochastic recursive gradient.
\newblock In \emph{International Conference on Machine Learning}, volume~70,
  pp.\  2613--2621, 2017.

\bibitem[Noble et~al.(2022)Noble, Bellet, and
  Dieuleveut]{noble2022differentially}
Noble, M., Bellet, A., and Dieuleveut, A.
\newblock Differentially private federated learning on heterogeneous data.
\newblock In \emph{International Conference on Artificial Intelligence and
  Statistics}, volume 151, pp.\  10110--10145. PMLR, 2022.

\bibitem[Samarati \& Sweeney(1998)Samarati and Sweeney]{samarati1998protecting}
Samarati, P. and Sweeney, L.
\newblock Protecting privacy when disclosing information: $k$-anonymity and its
  enforcement through generalization and suppression.
\newblock Technical report, SRI International, 1998.

\bibitem[Shang et~al.(2021)Shang, Xu, Liu, Liu, Shen, and
  Gong]{shang2021differentially}
Shang, F., Xu, T., Liu, Y., Liu, H., Shen, L., and Gong, M.
\newblock Differentially private {ADMM} algorithms for machine learning.
\newblock \emph{IEEE Transactions on Information Forensics and Security},
  16:\penalty0 4733--4745, 2021.

\bibitem[Shokri et~al.(2017)Shokri, Stronati, Song, and
  Shmatikov]{shokri2017membership}
Shokri, R., Stronati, M., Song, C., and Shmatikov, V.
\newblock Membership inference attacks against machine learning models.
\newblock In \emph{IEEE Symposium on Security and Privacy}, pp.\  3--18, 2017.

\bibitem[Song et~al.(2013)Song, Chaudhuri, and Sarwate]{song2013stochastic}
Song, S., Chaudhuri, K., and Sarwate, A.~D.
\newblock Stochastic gradient descent with differentially private updates.
\newblock In \emph{IEEE Global Conference on Signal and Information
  Processing}, pp.\  245--248, 2013.

\bibitem[Tran \& Cutkosky(2022)Tran and Cutkosky]{tran2022momentum}
Tran, H. and Cutkosky, A.
\newblock Momentum aggregation for private non-convex {ERM}.
\newblock In \emph{Advances in Neural Information Processing Systems}, 2022.
\newblock to appear.

\bibitem[Triastcyn \& Faltings(2019)Triastcyn and
  Faltings]{triastcyn2019federated}
Triastcyn, A. and Faltings, B.
\newblock Federated learning with bayesian differential privacy.
\newblock In \emph{2019 IEEE International Conference on Big Data}, pp.\
  2587--2596, 2019.

\bibitem[Wang et~al.(2020{\natexlab{a}})Wang, Gu, Boedihardjo, Wang, Barekat,
  and Osher]{wang2019dp}
Wang, B., Gu, Q., Boedihardjo, M., Wang, L., Barekat, F., and Osher, S.~J.
\newblock {DP-LSSGD: A} stochastic optimization method to lift the utility in
  privacy-preserving {ERM}.
\newblock In \emph{Mathematical and Scientific Machine Learning Conference},
  volume 107, pp.\  328--351, 2020{\natexlab{a}}.

\bibitem[Wang et~al.(2017)Wang, Ye, and Xu]{wang2017differentially}
Wang, D., Ye, M., and Xu, J.
\newblock Differentially private empirical risk minimization revisited: Faster
  and more general.
\newblock In \emph{Advances in Neural Information Processing Systems},
  volume~30, 2017.

\bibitem[Wang et~al.(2019{\natexlab{a}})Wang, Chen, and
  Xu]{wang2019differentially}
Wang, D., Chen, C., and Xu, J.
\newblock Differentially private empirical risk minimization with non-convex
  loss functions.
\newblock In \emph{International Conference on Machine Learning}, volume~97,
  pp.\  6526--6535, 2019{\natexlab{a}}.

\bibitem[Wang et~al.(2020{\natexlab{b}})Wang, Xiao, Devadas, and
  Xu]{wang2020differentially}
Wang, D., Xiao, H., Devadas, S., and Xu, J.
\newblock On differentially private stochastic convex optimization with
  heavy-tailed data.
\newblock In \emph{International Conference on Machine Learning}, volume 119,
  pp.\  10081--10091, 2020{\natexlab{b}}.

\bibitem[Wang et~al.(2019{\natexlab{b}})Wang, Jayaraman, Evans, and
  Gu]{wang2019efficient}
Wang, L., Jayaraman, B., Evans, D., and Gu, Q.
\newblock Efficient privacy-preserving stochastic nonconvex optimization.
\newblock \emph{arXiv preprint arXiv:1910.13659}, 2019{\natexlab{b}}.

\bibitem[Wang et~al.(2019{\natexlab{c}})Wang, Balle, and
  Kasiviswanathan]{wang2019subsampled}
Wang, Y.-X., Balle, B., and Kasiviswanathan, S.~P.
\newblock Subsampled r{\'e}nyi differential privacy and analytical moments
  accountant.
\newblock In \emph{International Conference on Artificial Intelligence and
  Statistics}, volume~89, pp.\  1226--1235, 2019{\natexlab{c}}.

\bibitem[Wang et~al.(2019{\natexlab{d}})Wang, Song, Zhang, Song, Wang, and
  Qi]{wang2019beyond}
Wang, Z., Song, M., Zhang, Z., Song, Y., Wang, Q., and Qi, H.
\newblock Beyond inferring class representatives: User-level privacy leakage
  from federated learning.
\newblock In \emph{IEEE INFOCOM Conference on Computer Communications}, pp.\
  2512--2520, 2019{\natexlab{d}}.

\bibitem[Yang et~al.(2019)Yang, Zhang, Chang, and Liang]{yang2019neural}
Yang, Z., Zhang, J., Chang, E.-C., and Liang, Z.
\newblock Neural network inversion in adversarial setting via background
  knowledge alignment.
\newblock In \emph{ACM SIGSAC Conference on Computer and Communications
  Security}, pp.\  225--240, 2019.

\bibitem[Yang et~al.(2021)Yang, Lei, Lyu, and Ying]{yang2021stability}
Yang, Z., Lei, Y., Lyu, S., and Ying, Y.
\newblock Stability and differential privacy of stochastic gradient descent for
  pairwise learning with non-smooth loss.
\newblock In \emph{International Conference on Artificial Intelligence and
  Statistics}, volume 130, pp.\  2026--2034. PMLR, 2021.

\bibitem[Yeom et~al.(2018)Yeom, Giacomelli, Fredrikson, and
  Jha]{yeom2018privacy}
Yeom, S., Giacomelli, I., Fredrikson, M., and Jha, S.
\newblock Privacy risk in machine learning: Analyzing the connection to
  overfitting.
\newblock In \emph{IEEE Computer Security Foundations Symposium}, pp.\
  268--282, 2018.

\bibitem[Zhang et~al.(2017)Zhang, Zheng, Mou, and Wang]{zhang2017efficient}
Zhang, J., Zheng, K., Mou, W., and Wang, L.
\newblock Efficient private {ERM} for smooth objectives.
\newblock In \emph{International Joint Conference on Artificial Intelligence},
  pp.\  3922--3928, 2017.

\bibitem[Zhang et~al.(2020{\natexlab{a}})Zhang, Fang, Liu, and
  Zhu]{zhang2020private}
Zhang, X., Fang, M., Liu, J., and Zhu, Z.
\newblock Private and communication-efficient edge learning: a sparse
  differential gaussian-masking distributed {SGD} approach.
\newblock In \emph{International Symposium on Theory, Algorithmic Foundations,
  and Protocol Design for Mobile Networks and Mobile Computing}, pp.\
  261--270, 2020{\natexlab{a}}.

\bibitem[Zhang et~al.(2020{\natexlab{b}})Zhang, Jia, Pei, Wang, Li, and
  Song]{zhang2020secret}
Zhang, Y., Jia, R., Pei, H., Wang, W., Li, B., and Song, D.
\newblock The secret revealer: Generative model-inversion attacks against deep
  neural networks.
\newblock In \emph{IEEE/CVF Conference on Computer Vision and Pattern
  Recognition}, pp.\  253--261, 2020{\natexlab{b}}.

\bibitem[Zhou et~al.(2020)Zhou, Chen, Hong, Wu, and Banerjee]{zhou2020private}
Zhou, Y., Chen, X., Hong, M., Wu, Z.~S., and Banerjee, A.
\newblock Private stochastic non-convex optimization: Adaptive algorithms and
  tighter generalization bounds.
\newblock \emph{arXiv preprint arXiv:2006.13501}, 2020.

\bibitem[Zhu et~al.(2019)Zhu, Liu, and Han]{zhu2019deep}
Zhu, L., Liu, Z., and Han, S.
\newblock Deep leakage from gradients.
\newblock In \emph{Advances in Neural Information Processing Systems},
  volume~32, 2019.

\end{thebibliography}
\bibliographystyle{icml2023}

\clearpage

\appendix
\onecolumn

\section{Supplementary Material for Numerical Experiments}
\label{app_sec: experiment}
In this section, we provide additional information and results on our numerical experiments that are not given in the main paper due to the space limitation. 

\subsection*{Hyper-Parameter Tuning}
The hyper-parameter of non-private GD was only learning rate $\eta$. The ones of DP-GD were learning rate $\eta$ and clipping radius $C_1$. The ones of Diff2-GD were learning rate $\eta$, the restart interval $T$ and clipping radius $C_1$ and $C_2$. 

For tuning clipping radius and restart interval, we ran DP-GD with $C_1 \in \{1, 3.0, 10.0, 30.0, 100.0\}$ and ran DIFF2-GD with $C_1, C_2 \in \{1, 3.0, 10.0, 30.0, 100.0\}$ and $T \in \{0.003R, 0.01R, 0.03R, 0.1R\}$ for $2,000$ rounds. For tuning learning rate $\eta$, we ran each implemented algorithm with $\eta \in \{0.5^i | i \in \{0, \ldots, 9\}\}$. To reduce the execution time, train loss was evaluated every $20$ rounds and the learning was stopped if the used learning rate was deemed inappropriate by checking the train loss. Specifically, the learning rate was determined to be inappropriate (i) if the train loss reached a NAN value; or (ii) if the patience count reached five. Here, the patience count was increased by one if the current train loss was greater than $1.05$ times the previous minimum train loss in the current learning, and reset to be zero if the current train loss was smaller than the previous best train loss. If we reached $2,000$ rounds using a certain learning rate $\eta_*$, the learning rate tuning was terminated and the best learning rate was determined to be $\eta_*$ instead of further trying smaller learning rates to avoid wasting  computational resources.

For two criteria of train loss and squared train gradient norm, we chose the the best hyper-parameters that minimized the criteria and report the learning results with the best hyper-parameters. For test loss, we chose the the best hyper-parameters that minimized train loss and report the learning results with the best hyper-parameters.

These procedures were independently executed for all the five random seeds. 

\subsection*{Datasets Information}
The information of three datasets used in our experiments is summarized in Table \ref{app_tab: datasets}.

\begin{table}[h]
  \caption{The summary of the datasets used in the numerical experiments.}
  \label{app_tab: datasets}
  \centering
  \begin{tabular}{ccc} \hline
      Dataset & \# Samples & \# Attributes \\ \hline \hline
      California Housing & $20,640$ & $8$ \\
      Gas Turbine CO and NOx Emission & $36,733$ & $9$\\
      BlogFeedback & $60,021$ & $280$ \\
      Cover Type & $58,1012$ & $54$ \\
      KDDCup99 & $4,898,431$ \footnotemark  & $41$ \\
      \hline
  \end{tabular}
\end{table}
\footnotetext{In our experiments, we used only $10$ percent of the data. }

\subsection*{Additional Numerical Results on Regression Tasks}
Here, we provide additional numerical results on the case $\varepsilon_\mathrm{DP} = 5.0$ on regression tasks. We confirmed that the results were similar to the ones when $\varepsilon_\mathrm{DP} = 3.0$ provided in the main paper.

\begin{figure*}[t]
\begin{subfigmatrix}{3}
\subfigure[Train Loss]{\includegraphics[width=5.5cm]{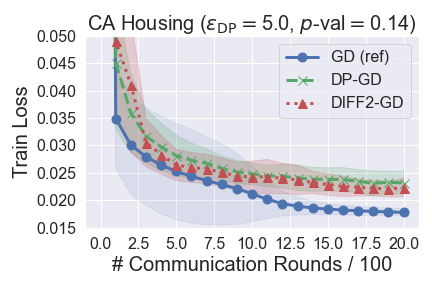}
}
\subfigure[Squared Train Grad. Norm]{\includegraphics[width=5.5cm]{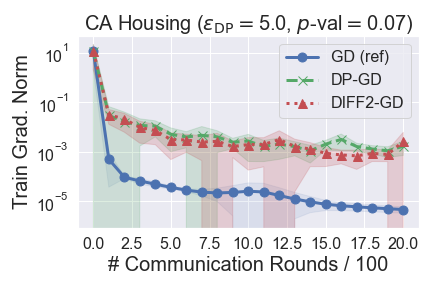}
}
\subfigure[Test Loss]
{\includegraphics[width=5.5cm]{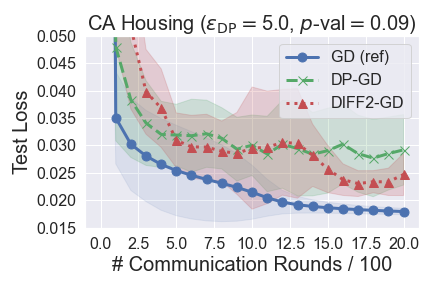}
}
\end{subfigmatrix}
\vspace{-0.5em}
\begin{subfigmatrix}{3}
\subfigure[Train Loss]{\includegraphics[width=5.5cm]{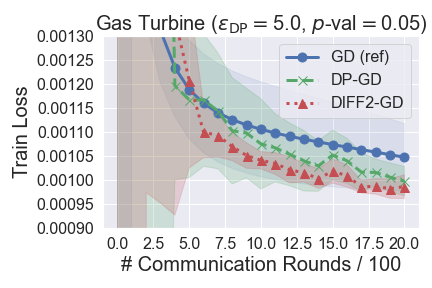}
}
\subfigure[Squared Train Grad. Norm]{\includegraphics[width=5.5cm]{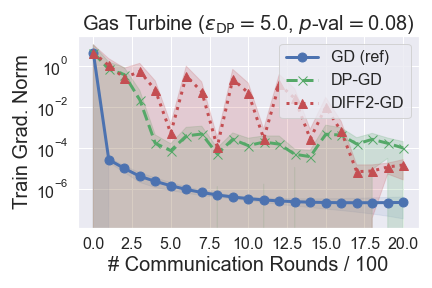}
}
\subfigure[Test Loss]
{\includegraphics[width=5.5cm]{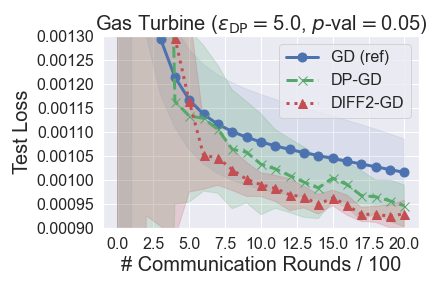}
}
\end{subfigmatrix}
\vspace{-0.5em}
\begin{subfigmatrix}{3}
\subfigure[Train Loss]{\includegraphics[width=5.5cm]{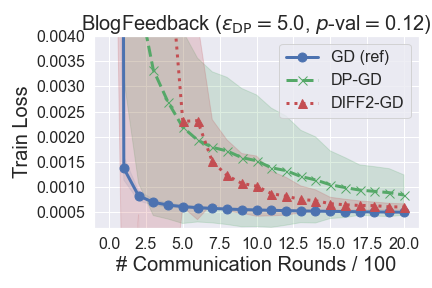}
}
\subfigure[Squared Train Grad. Norm]{\includegraphics[width=5.5cm]{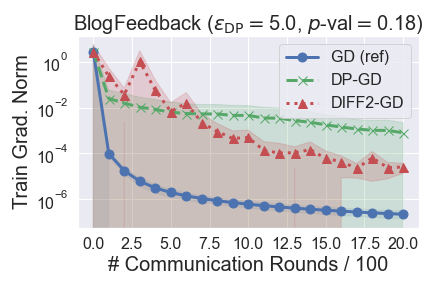}
}
\subfigure[Test Loss]
{\includegraphics[width=5.5cm]{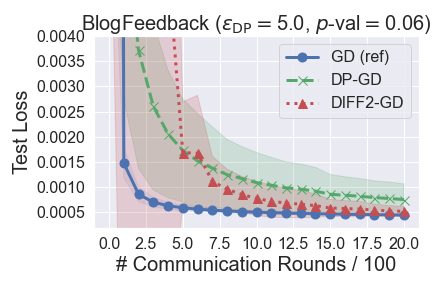}
}
\vspace{-0.5em}
\end{subfigmatrix}
\caption{Comparison of the train loss, train gradient norm and test loss against the number of communication rounds ($\varepsilon=5.0$, $\delta=10^{-5}$ and $R = 2,000$). (a)-(c) shows the comparison of the three criteria on California Housing dataset, (d)-(f) do the ones on Gas Turbine CO and NOx Emission dataset and (g)-(i) do the ones on BlogFeedback dataset. We confirm that DIFF2-GD consistently outperformed DP-GD. }
\label{fig: by_rounds}
\end{figure*}

\subsection*{Computing Infrastructures}
\begin{itemize}
    \item OS: Ubuntu 16.04.6
    \item CPU: AMD EPYC 7552 48-Core Processor.
    \item CPU Memory: 1.0 TB.    
    \item Programming language: Python 3.9.12.
    \item Deep learning framework: Pytorch 1.12.1.
\end{itemize}

\subsection*{Additional Numerical Results on Classification Tasks}

Here, we provide additional numerical results on the case $\varepsilon_\mathrm{DP} = 3.0$ on classification tasks. We confirm that the consistent superiority of the proposed method was observed even for classification tasks. 

\begin{figure}[H]
\begin{subfigmatrix}{3}
\subfigure[Train Loss]{\includegraphics[width=4.8cm]{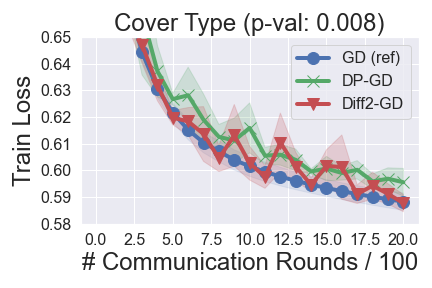}
}
\subfigure[Squared Train Grad. Norm]{\includegraphics[width=4.8cm]{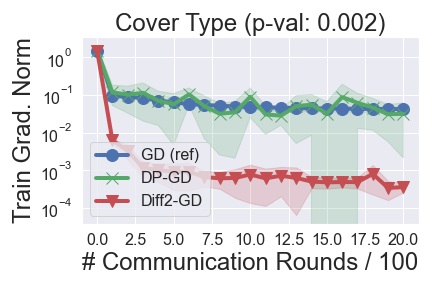}
}
\subfigure[Train Acc]
{\includegraphics[width=4.8cm]{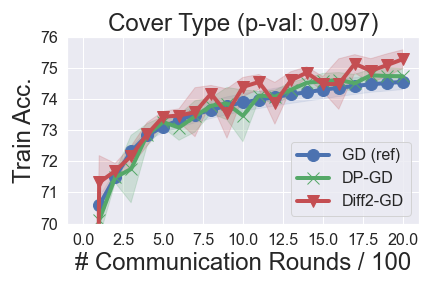}
}
\end{subfigmatrix}
\begin{subfigmatrix}{2}
\subfigure[Test Loss]{\includegraphics[width=4.8cm]{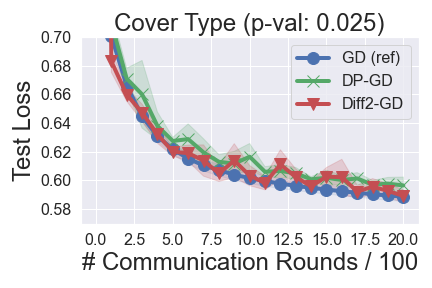}
}
\subfigure[Test Acc]{\includegraphics[width=4.8cm]{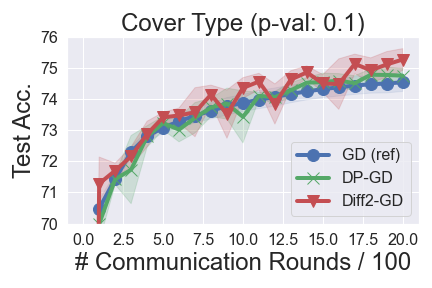}
}
\end{subfigmatrix}
\begin{subfigmatrix}{3}
\subfigure[Train Loss]{\includegraphics[width=4.8cm]{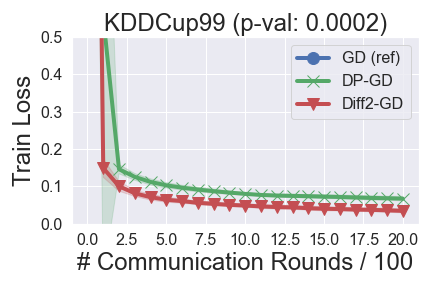}
}
\subfigure[Squared Train Grad. Norm]{\includegraphics[width=4.8cm]{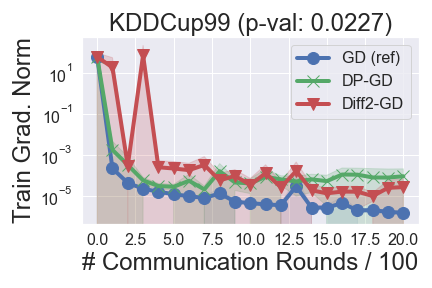}
}
\subfigure[Train acc]
{\includegraphics[width=4.8cm]{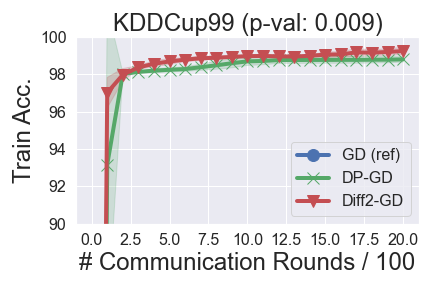}
}
\end{subfigmatrix}
\begin{subfigmatrix}{2}
\subfigure[Test Loss]{\includegraphics[width=4.8cm]{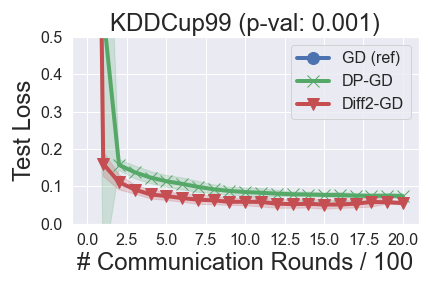}
}
\subfigure[Test Acc]{\includegraphics[width=4.8cm]{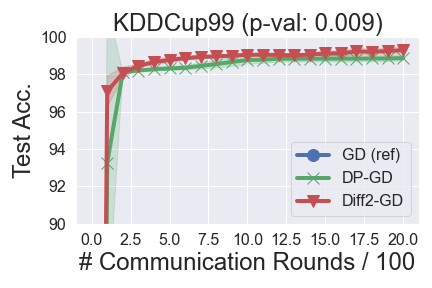}
}
\end{subfigmatrix}
\caption{Comparison of the train loss, train gradient norm and test loss against the number of communication rounds ($\varepsilon=3.0$, $\delta=10^{-5}$ and $R = 2,000$). (a)-(e) shows the comparison of the three criteria on Cover Type  dataset and (f)-(j) do the ones on KDDCup99  dataset. We confirm that DIFF2-GD consistently outperformed DP-GD. }
\label{app_fig: by_rounds}
\end{figure}

\section{Review of R\'enyi Differential Privacy}
Our DP analysis relies on  R\'enyi differential privacy (RDP) technique with subsampling and shuffling amplification. In this section, we briefly review some known results about RDP used in our analysis. 

\begin{definition}[$(\alpha, \varepsilon)$-RDP (\cite{mironov2017renyi})]
A randomized mechanism $\mathcal M: \mathcal D \to \mathcal R$ satisfies $(\alpha, \varepsilon)$-RDP ($\alpha \in (1, \infty)$ and $\varepsilon > 0$) if for any datasets $D, D' \in \mathcal D$ with $d_\mathrm{H}(D, D') = 1$, it holds that 

\begin{align*}
    \frac{1}{\alpha-1}\log\mathbb{E}_{o \sim \mathcal M(D')}\left(\frac{\mathcal M(D)(o)}{\mathcal M(D')(o)}\right)^\alpha \leq \varepsilon,
\end{align*}
where $M(D)(o)$ denotes the density of $\mathcal M(D)$ at $o$.
\end{definition}

\begin{lemma}[Post-processing Property of RDP (\cite{mironov2017renyi})]
\label{app_lem: post_processing}
Let $\mathcal M: \mathcal D \to \mathcal R$ be $(\alpha, \varepsilon)$-RDP and $g: \mathcal R \to \mathcal R'$ be any function. Then, $g\circ \mathcal M: \mathcal D \to \mathcal R'$ is also $(\alpha, \varepsilon)$-RDP. 
\end{lemma}

\begin{lemma}[Composition of RDP Mechanisms (\cite{mironov2017renyi})]
\label{app_lem: composition}
Let $\mathcal M_r: \mathcal R_1 \times \cdots \times \mathcal R_{r-1} \times \mathcal D \to \mathcal R_r$ be $(\alpha, \varepsilon_r)$-RDP for $r \in [R]$. Then, $\mathcal M: \mathcal D \to \mathcal R_1 \times \cdots \times \mathcal R_R$ defined by $\mathcal M(D) := (\mathcal M_1(D), \mathcal M_2(\mathcal M_1(D), D), \ldots, \mathcal M_R(\mathcal M_1(D), \ldots, D))$ is $(\alpha, \sum_{r=1}^R\varepsilon_r)$-RDP.
\end{lemma}

\begin{lemma}[From RDP to DP (\cite{mironov2017renyi})]
\label{app_lem: rdp_to_dp}
If a randomized mechanism $\mathcal M$ is $(\alpha, \varepsilon)$-RDP, then $\mathcal M$ is $(\varepsilon + \log(1/\delta)/(\alpha-1), \delta)$-DP for every $\delta \in (0, 1)$. 
\end{lemma}

\begin{definition}[$l_q$-sensitivity]
$\Delta_q(h) := \sup_{d_\mathrm H (D, D') = 1}\|h(D) - h(D')\|_q$ is called $l_q$-sensitivity for function $h$, where the maximum is taken over any adjacent datasets $D, D' \in \mathcal D$.
\end{definition}

\begin{lemma}[Gaussian Mechanism (\cite{mironov2017renyi})]
\label{app_lem: gaussian_mechanism}
Given a function $h$, Gaussian Mechanism $\mathcal M(D) := h(D) + \mathcal N(0, \sigma_1^2I)$ satisfies $(\alpha, \alpha\Delta_2^2(h)/(2\sigma_1^2))$-RDP for every $\alpha \in (1, \infty)$. 
\end{lemma}

\begin{lemma}[Subsampling  Amplification (Theorem 9 in \cite{wang2019subsampled})]
\label{app_lem: subsampling_amplification}
Let $\mathcal M$ be a randomized mechanism that takes a dataset of $b \leq n$ points as an input and $\gamma := b/n$. $\mathcal M \circ \mathrm{subsample}_\gamma$ be defined as: (1) $\mathrm{subsample}_\gamma$: subsample $\gamma n$ points without replacement from the input dataset with size $n$, and (2) apply $\mathcal M$ taking the subsampled points as the input. For every integer $\alpha \geq 2$, if $\mathcal M$ is $(\alpha, \varepsilon(\alpha))$-RDP, $\mathcal M \circ \mathrm{subsample}_\gamma$ is $(\alpha, \varepsilon'(\alpha))$-RDP, where
\begin{align*}
    \varepsilon'(\alpha) \leq&\ \frac{1}{\alpha-1}\log\left(1 + \gamma^2\binom{\alpha}{2}\min\left\{4(e^{\varepsilon(2)}-1), e^{\varepsilon(2)}\min\{2, (e^{\varepsilon(\infty)} - 1)^2\}\right\} \right. \\
    &+ \left. \sum_{j=3}^\alpha \gamma^j \binom{\alpha}{j}e^{(j-1)\varepsilon(j)}\min\{2, (e^{\varepsilon(\infty)}-1)^j\}
    \right).
\end{align*}
\end{lemma}

We can derive a simple upper bound using Lemma \ref{app_lem: subsampling_amplification}.
\begin{lemma}[Subsampling Amplification (Simple Upper Bound)]
\label{app_lem: subsampling_amplification_upper_bound}
On the same settings as in Lemma \ref{app_lem: subsampling_amplification}, if $\varepsilon(\alpha)$ is monotonically increasing with respect to $\alpha$, it holds that 
\begin{align*}
    \varepsilon'(\alpha) 
    \leq&\ \frac{2}{3}\left(4 + \frac{e}{c}\right)\frac{\gamma^2\alpha^2\varepsilon(2)}{\alpha-1}
\end{align*}
under $\varepsilon(\alpha)\leq 1/3 \wedge \log(1/(2\gamma\alpha))$ and $\gamma \leq \varepsilon(2)/(c\alpha)$ for some $c > 0$.

\end{lemma}
\begin{proof}
From Lemma \ref{app_lem: subsampling_amplification}, we have
\begin{align*}
    \varepsilon'(\alpha) \leq&\ \frac{1}{\alpha-1}\log\left(1 + 4\gamma^2\binom{\alpha}{2}(e^{\varepsilon(2)} - 1)  + 2\sum_{j=3}^\alpha \gamma^j \binom{\alpha}{j}e^{(j-1)\varepsilon(j)}
    \right).
\end{align*}
Suppose that $\varepsilon(\alpha) \leq \log(1/(2\gamma \alpha))$. Then, we can see that $\gamma \alpha e^{\varepsilon(\alpha)} < 1/2$. 
The third term can be bounded as follows:
\begin{align*}
    2\sum_{j=3}^\alpha \gamma^j \binom{\alpha}{j}e^{(j-1)\varepsilon(j)}  \leq&\ \frac{1}{3}\sum_{j=3}^\alpha (\gamma \alpha e^{\varepsilon(\alpha)})^j \\
    \leq&\ \frac{1}{3}\frac{(\gamma \alpha e^{\varepsilon(\alpha)})^3}{1 - \gamma \alpha e^{\varepsilon(\alpha)}} \\
    \leq&\ \frac{2(\gamma \alpha e^{\varepsilon(\alpha)})^3}{3}
\end{align*}

Assume that $\gamma  \leq \varepsilon(2)/(c\alpha)$ for some $c > 0$. Then, we have
\begin{align*}
    2\sum_{j=3}^\alpha \gamma^j \binom{\alpha}{j}e^{(j-1)\varepsilon(j)}
    \leq&\ \frac{2\gamma^2 \alpha^2 e^{3\varepsilon(\alpha)}\varepsilon(2)}{3c}.
\end{align*}
Hence, we obtain
\begin{align*}
    \varepsilon'(\alpha) \leq&\ \frac{1}{\alpha-1}\log\left(1 + \left(2(e^{\varepsilon(2)}-1) + \frac{2e^{3\varepsilon(\alpha)}\varepsilon(2)}{3c}\right)\gamma^2\alpha^2 
    \right) \\
    \leq&\ \left(2(e^{\varepsilon(2)}-1) + \frac{2e^{3\varepsilon(\alpha)}\varepsilon(2)}{3c}\right)\frac{\gamma^2\alpha^2}{\alpha-1} \\
    \leq&\ \left(2\varepsilon(2)(1+\varepsilon(2)) + \frac{2e\varepsilon(2)}{3c}\right)\frac{\gamma^2\alpha^2}{\alpha-1}.
\end{align*}
Here, for the second inequality, we used $\log(1+x) \leq x$. The last inequality holds because $e^x - 1 \leq x + x^2$ for $x \leq 3/2$ by assuming $\varepsilon(2) \leq\varepsilon(\alpha) \leq 1/3 \leq 3/2$. This gives the desired result. 
\end{proof}

\section{Differential Privacy Analysis of DIFF2-GD}
\label{app_sec: dp_analysis_diff2_gd}
In this section, we investigate differential privacy level of DIFF2 (Algorithm \ref{alg: DIFF2}) with GD-Routine (Algorithm \ref{alg: gd_routine}). 

First, we consider the $l_2$-sensitivity of $v_r$ for $r \in [R]$ with respect to $D_p$. 
\begin{lemma}[$l_2$-Sensitivity of $v_{r}$]
\label{app_lem: l2_sensitivity_full_grad}
In Algorithm \ref{alg: DIFF2}, $v_{1}$ has $l_2$-sensitivity $2 C_1/(n_pP)$ with respect to $D_p$. Furthermore, given the outputs of the previous mechanisms $\{x_{r'}, \widetilde v_{r'}\}_{r' \in [r-1]}$, $v_{r}$ has $l_2$-sensitivity $2 C_{2, r}/(n_pP)$ with respect to $D_p$. 
\end{lemma}
\begin{proof}
When $(r-1)\% T = 0$, the $l_2$-sensitivity of $v_{1} = (1/P)\sum_{p=1}^P d_{1}^{(p)}$ for adjacent local datasets $D_p$ and $D_p'$ can be bounded as 
\begin{align*}
    \frac{1}{n_p P}\left\| \min\left\{\frac{C_1}{\|\nabla \ell(x_{r-1}, z)\|}, 1\right\} \nabla \ell(x_{r-1}, z) -  \min\left\{\frac{C_1}{\|\nabla \ell(x_{r-1}, z')\|}, 1\right\} \nabla \ell(x_{r-1}, z')\right\|
    \leq \frac{2C_1}{n_p P}
\end{align*}
for some $z \neq z'$ from the definition of $d_{1}^{(p)}$ since $x_{r-1}$ is fixed. 

When $(r-1)\% T \neq 0$, the $l_2$-sensitivity of $v_{r} = (1/P)\sum_{p=1}^P d_r^{(p)} + \widetilde v_{r-1}$ for adjacent local datasets $D_p$ and $D_p'$ can be bounded as 
\begin{align*}
    &\frac{1}{n_p P}\left\| \min\left\{\frac{C_{2, r}}{\|\nabla \ell(x_{r-1}, z) - \nabla \ell(x_{r-2}, z)\|}, 1\right\} (\nabla \ell(x_{r-1}, z) - \nabla \ell(x_{r-2}, z) \right.\\
    &- \left.  \min\left\{\frac{C_{2, r}}{\|\nabla \ell(x_{r-1}, z') - \nabla \ell(x_{r-2}, z')\|}, 1\right\} (\nabla \ell(x_{r-1}, z') - \nabla \ell(x_{r-2}, z')\right\|
    \leq \frac{2C_{2, r}}{n_pP}
\end{align*}
for some $z \neq z'$ from the definition of $d_{r}^{(p)}$ since $x_{r-1}$, $x_{r-2}$ and $\widetilde v_{r-1}$ are fixed.
This finishes the proof. 
\end{proof}

Combined Lemma \ref{app_lem: l2_sensitivity_full_grad}  with Lemmas \ref{app_lem: gaussian_mechanism} and \ref{app_lem: composition}, we can show Proposition \ref{prop: dp_noise_level_diff2_gd_routine}. 

\subsection*{Proof of Proposition \ref{prop: dp_noise_level_diff2_gd_routine}}
First, we derive a RDP bound with respect to $D_p$ for each mechanism $\widetilde v_{k, r}$.

When $(r-1)\% T = 0$, from Lemma \ref{app_lem: gaussian_mechanism}, we know that $\widetilde v_{r}$ is $(\alpha, 2\alpha/(n_p^2P^2\sigma_1^2))$-RDP for every $\alpha \in (1, \infty)$ since $\xi_{1} \sim \mathcal N(0, \sigma_1^2C_1^2I)$ and $v_1$ has $l_2$-sensitivity $\Delta_2 = 2C_1/(n_pP)$ by the first part of Lemma \ref{app_lem: l2_sensitivity_full_grad}. 

Similarly, when $(r-1)\%T \neq 0$, from Lemma \ref{app_lem: gaussian_mechanism}, we know that $\widetilde v_{r}$ is $(\alpha, 2\alpha/(n_p^2P^2\sigma_2^2))$-RDP for every $\alpha \in (1, \infty)$ since $\xi_{r} \sim \mathcal N(0, \sigma_2^2C_{2, r}^2I)$ and $v_r$ has $l_2$-sensitivity $\Delta_2 = 2C_{2, r}/(n_pP)$ by the second part of Lemma \ref{app_lem: l2_sensitivity_full_grad}. 

Hence, given $\{x_{r'}, x_{r'}^\mathrm{out}, \widetilde v_{r'}\}_{r' \in [r-1]}$, $\{x_{r}, x_r^\mathrm{out}, \widetilde v_r\}$ are $(\alpha, 2\alpha/(n_p^2P^2\sigma^2))$-RDP from the post-processing property of RDP (Lemma \ref{app_lem: post_processing}) since $x_r = x_r^\mathrm{out} = x_{r-1} - \eta \widetilde v_r$, where $\sigma^2 = \sigma_1^2$ when $(r-1)\% T = 0$ and $\sigma^2 = \sigma_2^2$ when $(r-1)\%T \neq 0$. 

Now, by Lemma \ref{app_lem: composition}, it holds that mechanism $\{x_{r}, x_{r}^\mathrm{out}, \widetilde v_{r}\}_{r \in [R]}$ is $(\alpha, 2\alpha \lceil R/T\rceil/(n_p^2P^2\sigma_1^2) + 2\alpha (R -\lceil R/T\rceil) /(n_p^2P^2\sigma_2^2))$-RDP with respect to $D_p$.  

Therefore, setting $\alpha = 1 + \lceil 2\log(1/\delta_\mathrm{DP})/\varepsilon_\mathrm{DP}\rceil$, we only need to find $\sigma_1^2$ and $\sigma_2^2$ that satisfy $2\alpha \lceil R/T\rceil/(n_p^2P^2\sigma_1^2) + 2\alpha (R -\lceil R/T\rceil) /(n_p^2P^2\sigma_2^2)) \leq \varepsilon/2$. Then, we can see that for any constant $u > 1$, it is sufficient that $\sigma_1^2 = 4 u \alpha \lceil R/T\rceil /(n_\mathrm{min}^2P^2\varepsilon_\mathrm{DP})$ and $\sigma_2^2 = (4u/(u-1))\alpha (R - \lceil R/T \rceil)/(n_\mathrm{min}^2P^2\varepsilon_\mathrm{DP})$ guarantees $(\varepsilon_\mathrm{DP}, \delta_\mathrm{DP})$-DP of mechanism $\{x_{r}, x_{r}^\mathrm{out}, \widetilde v_{r}\}_{r \in [R]}$. This is the desired result.  

\section{Utility Analysis of DIFF2-GD}
\label{app_sec: utility_analysis_diff2_gd}
In this section, we give utility analysis of DIFF2 (Algorithm \ref{alg: DIFF2}) with GD-Routine (Algorithm \ref{alg: gd_routine}). 

\begin{lemma}[Descent Lemma]
\label{app_lem: descent_lemma_DIFF2}
Suppose that Assumption \ref{assump: local_loss_smoothness} holds. Under $\eta \leq 1/(2L)$, DIFF2 (Algorithm \ref{alg: DIFF2}) satisfies 
\begin{align*}
    \mathbb{E}[f(x_{r})] \leq&\ 
    \mathbb{E}[f(x_{r-1})] - \frac{\eta}{2}\mathbb{E}\|\nabla f(x_{r-1})\|^2 - \frac{1}{4\eta} \mathbb{E}\|x_{r} -  x_{r-1}\|^2 \\
    &+ \eta\mathbb{E}\left\|\frac{1}{\eta}(x_{r-1} - x_{r}) - \nabla f(x_{r-1})\right\|^2 .
\end{align*}
for every round $r \in [R]$. 
\end{lemma}
\begin{proof}
From $L$-smoothness of $f$, we have
\begin{align*}
    f(x_r) \leq&\ f(x_{r-1}) + \langle \nabla f(x_{r-1}), x_r - x_{r-1}\rangle + \frac{L}{2}\|x_r - x_{r-1}\|^2 \\
    =&\ f(x_{r-1}) +  \frac{1}{\eta}\left(- \frac{\eta^2}{2}\|\nabla f(x_{r-1})\|^2 - \frac{1}{2}\|x_{r} - x_{r-1}\|^2 + \frac{1}{2}\|x_r - x_{r-1} + \eta \nabla f(x_{r-1})\|^2\right)\\
    &+ \frac{L}{2}\|x_r - x_{r-1}\|^2 \\
    =&\ f(x_{r-1}) - \frac{\eta}{2}\|\nabla f(x_{r-1})\|^2 - \left(\frac{1}{2\eta} - \frac{L}{2}\right)\|x_{r} -  x_{r-1}\|^2 + \frac{\eta}{2}\left\|\frac{1}{\eta}(x_{r-1} - x_{r}) - \nabla f(x_{r-1})\right\|^2. 
\end{align*}
Taking expectation on both sides with respect to all the history of the randomness and assuming $\eta \leq 1/(2L)$ yield the desired result.
\end{proof}

\begin{proposition}
\label{app_prop: convergence_rate_DIFF2_gd_routine}
Suppose that Assumption \ref{assump: local_loss_smoothness}, \ref{assump: optimal_sol} and \ref{assump: local_loss_gradient_boundedness} hold. Assume that $C_1 \geq G$ and $C_2 \geq L$. Then, if we appropriately choose $\eta = \Theta(\min\{1/L, 1/\sqrt{T\sigma_2^2C_2^2d}\}$, DIFF2-GD satisfies
\begin{align*}
    \mathbb{E}\|\nabla f(x^\mathrm{out})\|^2 \leq&\ O\left(\frac{f(x_0) - f(x_*)}{\eta R} + \sigma_1^2C_1^2d \right).  
\end{align*}
\end{proposition}

\begin{proof}
Averaging the inequality in Lemma \ref{app_lem: descent_lemma_DIFF2} from $r=1$ to $R$, we know that
\begin{align}
    \mathbb{E}\|\nabla f(x^\mathrm{out})\|^2 = \frac{1}{R}\sum_{r=1}^R\mathbb{E}\|\nabla f(x_{r-1})\|^2 \leq&\ \frac{2(f(x_0) - f(x_*))}{\eta R} - \frac{1}{2\eta^2}\frac{1}{R}\sum_{r=1}^R\mathbb{E}\|x_{r} - x_{r-1}\|^2 \notag\\
    &+ \frac{2}{R}\sum_{r=1}^R \mathbb{E}  \left\|\frac{1}{\eta}(x_{r-1} - x_{r}) - \nabla f(x_{r-1})\right\|^2. \label{app_ineq: descent_lemma_average_DIFF2}  
\end{align}
Here, we used the existence of optimal solution $x_*$. 

Now, we consider the last term in (\ref{app_ineq: descent_lemma_average_DIFF2}). Since $x_r - x_{r-1} = - \eta\widetilde v_r$, we have
\begin{align*}
    \left\|\frac{1}{\eta}(x_{r-1} - x_{r}) - \nabla f(x_{r-1})\right\|^2 = \|\widetilde v_r - \nabla f(x_{r-1})\|^2.
\end{align*}
First, we consider the case $(r-1) \% T = 0$. 
Since $C_1 \geq G$ is assumed, it holds that $(1/P)\sum_{p=1}^P d_r^{(p)} = \nabla f(x_{r-1})$ and thus
$\widetilde v_r = \nabla f(x_{r-1}) + \xi_{r}$. Then, we have
\begin{align*}
    \mathbb{E}\|\widetilde v_r - \nabla f(x_{r-1})\|^2 =&\ \mathbb{E}\|\xi_r\|^2 = \sigma_1^2C_1^2d.
\end{align*}
Next, we consider the case $(r-1) \% T \neq 0$. 
Since $C_2 \geq L$ is assumed, it holds that $(1/P)\sum_{p=1}^P d_r^{(p)} = \nabla f(x_{r-1}) - \nabla f(x_{r-2})$ and thus $\widetilde v_r = \nabla f(x_{r-1}) - \nabla f(x_{r-2}) + \widetilde v_{r-1} + \xi_r$. Then, observe that
\begin{align*}
    \mathbb{E}\|\widetilde v_r - \nabla f(x_{r-1})\|^2 =&\ \mathbb{E}\|\widetilde v_{r-1} - \nabla f(x_{r-2}) + \xi_r\|^2 \\
    =&\ \mathbb{E}\|\widetilde v_{r-1} - \nabla f(x_{r-2})\|^2 + \mathbb{E}\|\xi_r\|^2 \\
    =&\ \sum_{r'=T(r)}^r \mathbb{E}\|\xi_{r'}\|^2 \\
    =&\ \sigma_1^2C_1^2d + \sigma_2^2d\sum_{r'=T(r)+1}^r C_{2, r}^2,  
\end{align*}
where $T(r)$ is the integer that satisfies $r+1-T \leq T(r) < r$ and $(T(r)-1) \% T = 0$. Here, we used the fact that $\widetilde v_{T(r)} = \nabla f(x_{T(r)-1}) + \xi_{T(r)}$. 
Using this relation, we get
\begin{align*}
    \frac{1}{R}\sum_{r=1}^R\mathbb{E}\|\widetilde v_r - \nabla f(x_{r-1})\|^2 =&\ \sigma_1^2C_1^2d + \sigma_2^2C_2^2d\frac{1}{R}\sum_{r=1}^R\sum_{r=T(r)+1}^r \|x_{r-1} - x_{r-2}\|^2 \\
    \leq&\ \sigma_1^2C_1^2d + \sigma_2^2C_2^2d\frac{1}{R}\sum_{r=1}^R\sum_{r=T(r)}^{T(r)+T-1} \|x_{r-1} - x_{r-2}\|^2 \\
    \leq&\ \sigma_1^2C_1^2d + T\sigma_2^2C_2^2d\frac{1}{R}\sum_{r=1}^R \|x_{r-1} - x_{r-2}\|^2. 
\end{align*}

Suppose that $\eta \leq 1/(2\sqrt{T\sigma_2^2C_2^2d})$. Then, it holds that
\begin{align*}
    - \frac{1}{2\eta^2}\frac{1}{R}\sum_{r=1}^R\mathbb{E}\|x_{r} - x_{r-1}\|^2 
    + \frac{2}{R}\sum_{r=1}^R \mathbb{E}  \left\|\frac{1}{\eta}(x_{r-1} - x_{r}) - \nabla f(x_{r-1})\right\|^2 \leq 2\sigma_1^2C_1^2 d.
\end{align*}

Applying this to (\ref{app_ineq: descent_lemma_average_DIFF2}), we obtain
\begin{align*}
    \mathbb{E}\|\nabla f(x^\mathrm{out})\|^2 \leq&\ O\left(\frac{f(x_0) - f(x_*)}{\eta R} + \sigma_1^2C_1^2d \right).  
\end{align*}
This is the desired result. 
\end{proof}

\subsection*{Proof of Theorem \ref{theorem: utility_DIFF2_gd_routine}}
From Proposition \ref{app_prop: convergence_rate_DIFF2_gd_routine}, under $C_1 = \Theta(G)$, we know that 
\begin{align*}
    \mathbb{E}\|\nabla f(x^\mathrm{out})\|^2 \leq&\ O\left(\frac{f(x_0) - f(x_*)}{\eta R} + \sigma_1^2G^2d \right).  
\end{align*}
Suppose that $f(x_0) - f(x_*) = O(1)$. 

By Proposition \ref{prop: dp_noise_level_diff2_gd_routine}, the necessary noise levels for $(\varepsilon_\mathrm{DP}, \delta_\mathrm{DP})$-DP with respect to $D_p$ is 
$\sigma_1^2 = \widetilde \Theta(R/(Tn_\mathrm{min}^2P^2 \varepsilon_\mathrm{DP}^2))$ and $\sigma_2^2 = \widetilde \Theta(R/(n_\mathrm{min}^2P^2\varepsilon_\mathrm{DP}^2))$ respectively. Using these results and substituting the definition of $\eta$, we have
\begin{align*}
    \mathbb{E}\|\nabla f(x^\mathrm{out})\|^2 \leq&\ O\left(\frac{1}{\eta R}\right) + \widetilde O\left(\frac{RG^2d}{Tn_\mathrm{min}^2P^2\varepsilon_\mathrm{DP}^2} \right) \\
    =&\ O\left( \frac{L + \sqrt{T}L\sigma_2\sqrt{d}}{R}\right) + \widetilde O\left( \frac{RG^2d}{Tn_\mathrm{min}^2P^2\varepsilon_\mathrm{DP}^2} \right) \\
    =&\ O\left( \frac{L}{R}\right) + \widetilde O\left( \frac{\sqrt{T}L\sqrt{d}}{n_\mathrm{min}P\varepsilon_\mathrm{DP}\sqrt{R}} + \frac{RG^2d}{Tn_\mathrm{min}^2P^2\varepsilon_\mathrm{DP}^2} \right).  
\end{align*}
Assume that $n_\mathrm{min} P =  \Omega( G^2\sqrt{d}/(L\varepsilon_\mathrm{DP}))$.
We define
\begin{align*}
    T := \Theta\left( 1 \vee  \left(\frac{G^2\sqrt{d}}{Ln_\mathrm{min}P\varepsilon_\mathrm{DP}}\right)^\frac{2}{3} R\right)
\end{align*}
with $1 \leq T \leq R$.

Then, we obtain
\begin{align*}
    \mathbb{E}\|\nabla f(x^\mathrm{out})\|^2 \leq&\  O\left( \frac{L}{R}\right) + \widetilde O\left(\frac{L\sqrt{d}}{n_\mathrm{min}P\varepsilon_\mathrm{DP}\sqrt{R}} +  \frac{(LGd)^\frac{2}{3}}{(n_\mathrm{min}P\varepsilon_\mathrm{DP})^\frac{4}{3}} \right).  
\end{align*}

Finally, setting 
\begin{align*}
    R = \Theta\left( 1 \vee \frac{L}{\varepsilon_\mathrm{opt}}\right) \vee \widetilde \Theta\left( \frac{L^2d}{n_\mathrm{min}^2P^2\varepsilon_\mathrm{DP}^2\varepsilon_\mathrm{opt}^2}\right)
\end{align*}
with $\varepsilon_\mathrm{opt} := \Theta\left(\frac{(LGd)^\frac{2}{3}}{(n_\mathrm{min}P\varepsilon_\mathrm{DP})^\frac{4}{3}}\right)$ 
gives the desired result.


\section{Differential Privacy Analysis of DIFF2-BVR-L-SGD}
In this section, we provide differential privacy analysis of DIFF2 (Algorithm \ref{alg: DIFF2}) with BVR-L-SGD-Routine (Algorithm \ref{alg: BVR-L-SGD_routine}). 
We investigate the minimum noise level $\sigma_1^2$,  $\sigma_2^2$ and $\sigma_3^2$ to satisfy $(\varepsilon_\mathrm{DP}, \delta_\mathrm{DP})$-DP with respect to $D_p$ for $p \in [P]$.

Since we can utilize Lemma \ref{app_lem: l2_sensitivity_full_grad} for the $l_2$ sensitivity of $v_r$, we focus on the $l_2$-sensitivity of $v_{k, r}$ for $k \in [K]$ and $r \in [R]$. 

\begin{lemma}[$l_2$-Sensitivity of $v_{k, r}$]
\label{app_lem: l2_sensitivity_BVR-L-SGD_routine}
Given the outputs of the previous mechanisms $\{x_{r-1}, \widetilde v_{r}\}\cup\{x_{k', r-1}\}_{k' \in [k-1]}$,
the $l_2$-sensitivity of $v_{k, r}$ defined in Algorithm \ref{alg: BVR-L-SGD_routine} is
\begin{align*}
    \begin{cases}
       \frac{2C_{3, k, r}}{b} & (p = p_r), \\
       0 & (p \neq p_r).
    \end{cases}
\end{align*}
with respect to $D_p$ for $k \geq 2$. 
\end{lemma}
\begin{proof}
The $l_2$-sensitivity of $v_{k, r}$ for adjacent local datasets $D_p$ and $D_p'$ can be bounded as 
\begin{align*}
    &\frac{1}{b}\left\| \min\left\{\frac{C_{3, k, r}}{\|\nabla \ell(x_{k-1, r-1}, z) - \nabla \ell(x_{k-2, r-1}, z)\|}, 1\right\} (\nabla \ell(x_{k-1, r-1}, z) - \nabla \ell(x_{k-2, r-1}, z) \right.\\
    &- \left.  \min\left\{\frac{C_{3, k, r}}{\|\nabla \ell(x_{k-1, r-1}, z') - \nabla \ell(x_{k-2, r-1}, z')\|}, 1\right\} (\nabla \ell(x_{k-1, r-1}, z') - \nabla \ell(x_{k-2, r-1}, z')\right\|
    \leq \frac{2C_{3, k, r}}{b}
\end{align*}
for some $z \neq z'$ from the definition of $ v_{k, r}^{(p)}$ since $x_{k-1, r-1}$, $x_{k-2, r-1}$ and $\widetilde v_{k-1, r}$ are fixed. When $p \neq p_r$, the sensitivity of $v_{k, r}$ with respect to $p$ is trivially zero. This finishes the proof. 
\end{proof}

Combined Lemmas \ref{app_lem: l2_sensitivity_full_grad} and \ref{app_lem: l2_sensitivity_BVR-L-SGD_routine} with Lemmas \ref{app_lem: gaussian_mechanism} and \ref{app_lem: composition}, we have the following Proposition:

\begin{proposition}[Minimum Noise Level for $(\varepsilon_\mathrm{DP}, \delta_\mathrm{DP})$-DP]
\label{app_prop: dp_noise_level_DIFF2_BVR-L-SGD_routine}
For every integer $\alpha \geq 2$, mechanism $\{x_{r}, x_{r}^\mathrm{out}, \widetilde v_{r}\}_{r \in [R]}$ defined in DIFF2-BVR-L-SGD is $(2\alpha \lceil R/T\rceil/(n_p^2P^2\sigma_1^2) + 2\alpha (R - \lceil R/T\rceil)/(n_p^2P^2\sigma_2^2) + K\lceil R/P\rceil \varepsilon'(\alpha) + \log(1/\delta_\mathrm{DP})/(\alpha-1), \delta_\mathrm{DP})$-DP, where $\varepsilon'(\alpha)$ is defined in Lemma \ref{app_lem: subsampling_amplification} with $\varepsilon(\alpha) := 2\alpha/(b^2\sigma_3^2)$ and $n_\mathrm{min} := \min\{n_p\}_{p=1}^P$. In particular, setting $\alpha := 1 + \lceil 2\log(1/\delta_\mathrm{DP})/\varepsilon_\mathrm{DP} \rceil$, 
\begin{align*}
    \sigma_1^2 = \frac{4u_1\alpha \lceil \frac{R}{T}\rceil}{n_\mathrm{min}^2P^2\varepsilon_\mathrm{DP}} \text{, } \sigma_2^2 = \frac{4u_2\alpha (R - \lceil \frac{R}{T}\rceil)}{n_\mathrm{min}^2P^2\varepsilon_\mathrm{DP}} 
    \text{ and }
    2K\left\lceil \frac{R}{P}\right\rceil\varepsilon'(\alpha) \leq \left(1 - \frac{1}{u_1} - \frac{1}{u_2}\right)\varepsilon_\mathrm{DP}
\end{align*} 
are a sufficient condition to guarantee $(\varepsilon_\mathrm{DP}, \delta_\mathrm{DP})$-DP for any constants $u_1, u_2 > 1$.

Furthermore, there exist $\sigma_1^2$, $\sigma_2^2$ and $\sigma_3^2$ which satisfy the above condition such that
\begin{align*}
    \sigma_1^2 = \widetilde \Theta\left(\frac{R}{Tn_\mathrm{min}^2 P^2\varepsilon_\mathrm{DP}^2}\right)
    \text{, }
    \sigma_2^2 = \widetilde \Theta\left(\frac{R}{n_\mathrm{min}^2 P^2\varepsilon_\mathrm{DP}^2}\right) 
    \text{ and }
    \sigma_3^2 = \widetilde \Theta\left( \frac{KR}{n_\mathrm{min}^2P\varepsilon_\mathrm{DP}^2} \vee \frac{1}{b^2\varepsilon_\mathrm{DP}}\right)
\end{align*}
under $b \leq \min\{n_\mathrm{min}/(2e\alpha),  (4n_\mathrm{min}/(\alpha\sigma_3^2))^{1/3}\}$.
\end{proposition}

\begin{proof}
First, from the arguments of Lemma \ref{app_lem: l2_sensitivity_full_grad}, we know that $\widetilde v_{1, r} = \widetilde v_{r}$ and thus $x_{1, r-1}$ is $(\alpha, 2\alpha/(n_p^2P^2\sigma_1^2))$-RDP when $(r-1)\% T = 0$ and $(\alpha, 2\alpha/(n_p^2P^2\sigma_2^2))$-RDP when $(r-1)\%T \neq 0$ for every $\alpha \in (1, \infty)$ given the outputs of the previous mechanisms. 

Next, if $p = p_{r}$, from Lemmas \ref{app_lem: gaussian_mechanism} and \ref{app_lem: l2_sensitivity_BVR-L-SGD_routine}, given a subsampled minibatch, it holds that $\widetilde v_{k, r}$ is $(\alpha, \varepsilon(\alpha))$-RDP for $k \geq 2$, where $\varepsilon(\alpha) := 2\alpha/(b^2\sigma_3^2)$ for every $\alpha \in (1, \infty)$ since $\xi_{2, r} \sim \mathcal N(0, \sigma_3^2C_{3, k, r}^2 I)$. Then, applying Lemma \ref{app_lem: subsampling_amplification}, we can see that $\widetilde v_{k, r}$ is $(\alpha, \varepsilon'(\alpha))$-RDP for integer $\alpha \geq 2$, where $\varepsilon'(\alpha)$ is defined in Lemma \ref{app_lem: subsampling_amplification} given the outputs of the previous mechanisms. This implies that $\{x_{k, r-1}, \widetilde v_{k, r}\}$ is also $(\alpha, \varepsilon'(\alpha))$-RDP for $k \geq 2$ from Lemma \ref{app_lem: post_processing} given the outputs of the previous mechanisms. 

When $p\neq p_r$, $\widetilde v_{k, r}$ and thus $x_{k, r-1}$ are trivially $(\alpha, 0)$-RDP with respect to $D_p$ for $k \geq 2$.

Now, by Lemma \ref{app_lem: composition}, it holds that mechanism $\{x_{k, r-1}\}_{k \in [K]\cup\{0\}, r \in [R]}\cup\{\widetilde v_{k, r}\}_{k \in [K], r \in [R]}$ is $(\alpha, 2\alpha \lceil R/T\rceil/(n_p^2P^2\sigma_1^2) + 2\alpha (R - \lceil R/T\rceil)/(n_p^2P^2\sigma_2^2) + K\lceil R/P\rceil \varepsilon'(\alpha))$-RDP with respect to $D_p$.

Then, using Lemma \ref{app_lem: rdp_to_dp}, we can see that mechanism $\{x_{k, r-1}\}_{k \in [K]\cup\{0\}, r \in [R]}\cup\{\widetilde v_{k, r}\}_{k \in [K], r \in [R]}$ is $(2\alpha \lceil R/T\rceil/(n_p^2P^2\sigma_1^2) + 2\alpha (R - \lceil R/T\rceil)/(n_p^2P^2\sigma_2^2) + K\lceil R/P\rceil \varepsilon'(\alpha) + \log(1/\delta_\mathrm{DP})/(\alpha-1), \delta_\mathrm{DP})$-DP for every integer $\alpha \geq 2$. Note that $\{x_{r}, x_{r}^\mathrm{out}, \widetilde v_{r}\}_{r \in [R]} \subset \{x_{k, r-1}\}_{k \in [K]\cup\{0\}, r \in [R]}\cup\{\widetilde v_{k, r}\}_{k \in [K], r \in [R]}$.

Now, choosing $\alpha := 1 + \lceil 2\log(1/\delta_\mathrm{DP})/\varepsilon_\mathrm{DP} \rceil$, 
\begin{align*}
    \sigma_1^2 = \frac{4u_1\alpha \lceil \frac{R}{T}\rceil}{n_\mathrm{min}^2P^2\varepsilon_\mathrm{DP}} 
    = \widetilde \Theta\left(\frac{R}{Tn_\mathrm{min}^2 P^2\varepsilon_\mathrm{DP}^2}\right), 
\end{align*}
\begin{align*}
    \sigma_2^2 = \frac{4u_2\alpha (R - \lceil \frac{R}{T}\rceil)}{n_\mathrm{min}^2P^2\varepsilon_\mathrm{DP}} 
    = \widetilde \Theta\left(\frac{R}{n_\mathrm{min}^2 P^2\varepsilon_\mathrm{DP}^2}\right), 
\end{align*}
and
\begin{align*}        
    2K\left\lceil \frac{R}{P}\right\rceil\varepsilon'(\alpha) \leq \left(1 - \frac{1}{u_1} - \frac{1}{u_2}\right)\varepsilon_\mathrm{DP}
\end{align*}
guarantees $(\varepsilon_\mathrm{DP}, \delta_\mathrm{DP})$ of mechanism $\{x_{r}, x_{r}^\mathrm{out}, \widetilde v_{r}\}_{r \in [R]}$ for any constants $u_1, u_2 > 1$.

Finally, we derive a simple upper bound of $\sigma_3^2$.
From Lemma \ref{app_lem: subsampling_amplification_upper_bound}, we know that
\begin{align*}
    \varepsilon'(\alpha) 
    \leq&\ \frac{2}{3}\left(4 + e\right)\frac{\gamma^2\alpha^2\varepsilon(2)}{\alpha-1}
\end{align*}
under $\varepsilon(\alpha) \leq 1/3 \wedge \log(1/(2\gamma\alpha))$ and $\gamma \leq \varepsilon(2)/\alpha$\footnotemark.
\footnotetext{Here, we set $c = 1$ in Lemma \ref{app_lem: subsampling_amplification_upper_bound} for simple presentations. Using $c < 1$ is beneficial to relax the condition of $b$ at the expense of the noise level.}

Suppose that 
$$\sigma_3^2 \geq \frac{8}{3}\left(4 + e\right)\frac{\alpha^2}{\alpha-1}\frac{6K\lceil \frac{R}{P}\rceil}{n_\mathrm{min}^2\varepsilon_\mathrm{DP}}\vee \frac{6\alpha}{b^2}$$
and
$$b \leq \frac{n_\mathrm{min}}{2e\alpha} \wedge \left(\frac{4n_\mathrm{min}}{\alpha\sigma_3^2}\right)^\frac{1}{3}. $$ 


Since $2\gamma\alpha \leq 1/e$, it holds that $\log(1/(2\gamma\alpha)) \geq 1$. Thus, $\varepsilon(\alpha) = 2\alpha/(b^2\sigma_3^2) \leq 1/3 \wedge \log(1/(2\gamma\alpha))$ is satisfied. Next, note that $\gamma \leq b/n_\mathrm{min} \leq 4/(\alpha b^2\sigma_3^2) = \varepsilon(2)/\alpha$ holds. Finally, 
\begin{align*}
 \varepsilon'(\alpha) 
    \leq&\ \frac{2}{3}\left(4 + e\right)\frac{\gamma^2\alpha^2\varepsilon(2)}{\alpha-1} \\
    \leq&\ \frac{8}{3}\left(4 + e\right)\frac{\alpha^2}{\alpha-1}\frac{1}{n_\mathrm{min}^2\sigma_3^2},
\end{align*}
and thus $6K\lceil R/P\rceil\varepsilon'(\alpha) \leq \varepsilon_\mathrm{DP}$ is satisfied. 

Therefore, 
\begin{align*}
    \sigma_3^2 \geq \widetilde \Theta\left( \frac{KR}{n_\mathrm{min}^2P\varepsilon_\mathrm{DP}^2} \vee \frac{1}{b^2\varepsilon_\mathrm{DP}}\right).
\end{align*}
is sufficient to hold $6K\lceil R/P\rceil\varepsilon'(\alpha) \leq \varepsilon_\mathrm{DP}$.

This finishes all the proof.

\end{proof}

\section{Utility Analysis of DIFF2 with BVR-L-SGD Routine}
In this section, we give utility analysis of DIFF2 (Algorithm \ref{alg: DIFF2}) with BVR-L-SGD-Routine (Algorithm \ref{alg: BVR-L-SGD_routine}). 

First, we focus on analysing BVR-L-SGD-Routine (Algorithm \ref{alg: BVR-L-SGD_routine}). 

\begin{lemma}[Descent Lemma]
\label{app_lem: descent_lemma}
Suppose that Assumption \ref{assump: local_loss_smoothness} holds. Under $\eta \leq 1/(2L)$, BVR-L-SGD-Routine satisfies
\begin{align*}
    \mathbb{E}[f(x_{k, r-1})] \leq&\ 
    \mathbb{E}[f(x_{k-1, r-1})] - \frac{\eta}{2}\mathbb{E}\|\nabla f(x_{k-1, r-1})\|^2 - \frac{1}{4\eta} \mathbb{E}\|x_{k, r-1} -  x_{k-1, r-1}\|^2 \\
    &+ \eta\mathbb{E}\|\widetilde v_{k, r} - \nabla f(x_{k-1, r-1})\|^2, 
\end{align*}
for every round $r \in [R]$ and iteration $k \in [K]$. 
\end{lemma}

\begin{proof}
For simple presentations, $x_k$, $\xi_k$, $p$, $v_k$ and $\widetilde v_k$ denote $x_{k, r-1}$, $\xi_{k, r}$, $p_r$ $v_{k, r}$ and $\widetilde v_{k, r}$ respectively in this proof. 

Let $\eta$ be $\eta$ for $k \in \{2, \ldots, K\}$ and be $\eta$ for $k=1$.

From $L$-smoothness of $f$, we have
\begin{align*}
    f(x_k) \leq&\ f(x_{k-1}) + \langle \nabla f(x_{k-1}), x_k - x_{k-1}\rangle + \frac{L}{2}\|x_k - x_{k-1}\|^2 \\
    =&\ f(x_{k-1}) +  \frac{1}{\eta}\left(- \frac{\eta^2}{2}\|\nabla f(x_{k-1})\|^2 - \frac{1}{2}\|x_{k} - x_{k-1}\|^2 + \frac{1}{2}\|x_k - x_{k-1} + \eta \nabla f(x_{k-1})\|^2\right)\\
    &+ \frac{L}{2}\|x_k - x_{k-1}\|^2 \\
    =&\ f(x_{k-1}) - \frac{\eta}{2}\|\nabla f(x_{k-1})\|^2 - \left(\frac{1}{2\eta} - \frac{L}{2}\right)\|x_{k} -  x_{k-1}\|^2 + \frac{\eta}{2}\|\widetilde v_k - \eta\nabla f(x_{k-1})\|^2. 
\end{align*}
Taking expectation on both sides with respect to the randomness at iteration $k$, 
\begin{align*}
    \mathbb{E}[f(x_k)] \leq&\ f(x_{k-1}) - \frac{\eta}{2}\|\nabla f(x_{k-1})\|^2 - \left(\frac{1}{2\eta} - \frac{L}{2}\right)\mathbb{E}\|x_{k} -  x_{k-1}\|^2 + \frac{\eta}{2}\mathbb{E}\|\widetilde v_k - \eta\nabla f(x_{k-1})\|^2. 
\end{align*}
Taking expectation with respect to all the history of the randomness, under $\eta \leq 1/(2L)$ we get

\begin{align*}
    \mathbb{E}[f(x_k)] \leq&\ 
    \mathbb{E}[f(x_{k-1})] - \frac{\eta}{2}\mathbb{E}\|\nabla f(x_{k-1})\|^2 - \frac{1}{4\eta} \mathbb{E}\|x_{k} -  x_{k-1}\|^2 + \frac{\eta}{2}\mathbb{E}\|\widetilde v_k - \nabla f(x_{k-1})\|^2.
\end{align*}
This is the desired result. 
\end{proof}

\begin{proposition}
\label{app_prop: local_opt}
Suppose that assumptions \ref{assump: similarity} and \ref{assump: local_loss_smoothness} hold. Then, if $C_3 \geq L$ with $C_3 = \Theta(L)$ and $\eta \leq \min\{1/(2L),  1/(2\sqrt{2e(K^2\eta^2 + KL^2/b+ KC_3^2\sigma_3^2d)})\} = \Theta(\min\{1/L, 1/(K\zeta), \sqrt{b}/(\sqrt{K}L), 1/(\sqrt{K}L\sigma_3\sqrt{d})\})$, BVR-L-SGD-Routine (Algorithm \ref{alg: BVR-L-SGD_routine}) satisfies
\begin{align*}
    \mathbb{E}\|\nabla f(x_{r}^\mathrm{out})\|^2 \leq \frac{2(\mathbb{E}[f(x_{r-1})] - \mathbb{E}[f(x_{r})])}{\eta K} - \frac{1}{4\eta^2}\frac{1}{K}\sum_{k=1}^K \mathbb{E}\|x_{k, r-1} - x_{k-1, r}\|^2 + 2e\mathbb{E}\|\widetilde v_r - \nabla f(x_{r-1})\|^2.
\end{align*}
for every round $r \in [R]$.
\end{proposition}
\begin{proof}
For simple presentations, $x_k$, $\xi_k$, $p$, $v_k$ and   $\widetilde v_k$denote $x_{k, r-1}$, $\xi_{k, r}$, $p_r$ $v_{k, r}$ and $\widetilde v_{k, r}$ respectively in this proof. 

Summing up the inequality in Lemma \ref{app_lem: descent_lemma} from $k=1$ to $K$ yields
\begin{align}
    \mathbb{E}\|\nabla f(x_{\hat k-1})\|^2 = \frac{1}{K}\sum_{k=1}^K \mathbb{E}\|\nabla f(x_{k-1})\|^2
    \leq&\ 
    \frac{2(\mathbb{E}[f(x_{0})] -  \mathbb{E}[f(x_K)])}{\eta K}  -  \frac{1}{2\eta^2}\frac{1}{K}\sum_{k=1}^K \mathbb{E}\|x_k - x_{k-1}\|^2 \notag \\
    &+ \frac{2}{K}\sum_{k=1}^K\mathbb{E}\|\widetilde v_k - \nabla f(x_{k-1})\|^2. \label{app_ineq: decent_lemma_sum}
\end{align}

Now, we bound $\mathbb{E}\|\widetilde v_k - \nabla f(x_{k-1})\|^2$ for $k \geq 2$. Since we assume $C_2 \geq L$, it always holds  $C_2\|x_{k-1} - x_{k-2}\|/\|\nabla \ell(x_{k-1}, z_l) - \nabla \ell(x_{k-2}, z_l)\| \geq 1$ under $L$-smoothness of $\ell$, and thus we have
\begin{align*}
    v_k = g_k(x_{k-1}) - g_k(x_{k-2}) + \widetilde v_{k-1},
\end{align*}
where $g_k(x_{k-1}) := \frac{1}{b}\sum_{z \in I_{k, r}} \nabla \ell(x_{k-1}, z)$ and $g_k(x_{k-2}) := \frac{1}{b}\sum_{z \in I_{k, r}} \nabla \ell(x_{k-2}, z)$.

Hence, we have
\begin{align*}
    &\mathbb{E}\|\widetilde v_k - \nabla f(x_{k-1})\|^2 \\
    =&\mathbb{E}\|v_k - \nabla f(x_{k-1})\|^2 + \mathbb{E}\|\xi_k\|^2 \\
    =&\ \mathbb{E}\|g_k(x_{k-1}) - g_k(x_{k-2}) + \widetilde v_{k-1} - \nabla f(x_{k-1})\|^2 + \mathbb{E}\|\xi_k\|^2\\
    =&\ \mathbb{E}\|\nabla f_{p}(x_{k-1}) - \nabla f_{p}(x_{k-2}) + \widetilde v_{k-1} - \nabla f(x_{k-1})\|^2 \\
    &+ \mathbb{E}\|g_k(x_{k-1}) - g_k(x_{k-2}) - \nabla f_{p}(x_{k-1}) + \nabla f_{p}(x_{k-2})\|^2  + \mathbb{E}\|\xi_k\|^2\\
    \leq&\ (1+1/K)\mathbb{E}\|\widetilde v_{k-1} - \nabla f(x_{k-2})\|^2 \\
    &+ (1+K)\mathbb{E}\|\nabla f_{p}(x_{k-1}) - \nabla f_{p}(x_{k-2}) - \nabla f(x_{k-1}) + \nabla f(x_{k-2})\|^2  \\
    &+ \mathbb{E}\|g_k(x_{k-1}) - g_k(x_{k-2}) - \nabla f_{p}(x_{k-1}) + \nabla f_{p}(x_{k-2})\|^2 \\
    &+ \mathbb{E}\|\xi_k\|^2.
\end{align*}

The second term can be bounded as follows:
\begin{align*}
    &\mathbb{E}\|\nabla f_{p}(x_{k-1}) - \nabla f_{p}(x_{k-2}) - \nabla f(x_{k-1}) + \nabla f(x_{k-2})\|^2 \\
    =&\ \mathbb{E}\|\nabla f_{p}(x_{k-1}) - \nabla f_{p}(x_{k-2}) - \nabla f(x_{k-1}) + \nabla f(x_{k-2})\|^2 \\
    =&\ \mathbb{E}\left\|\int_0^1 \nabla^2 f_{p}((1-\theta)x_{k-1} + \theta x_{k-2})(x_{k-1} - x_{k-2})d\theta - \int_0^1 \nabla^2 f((1-\theta)x_{k-1} + \theta x_{k-2})(x_{k-1} - x_{k-2})d\theta \right\|^2 \\
    \leq&\ \mathbb{E}\int_0^1\left\|(\nabla^2 f_{p}((1-\theta)x_{k-1} + \theta x_{k-2}) - \nabla^2 f((1-\theta)x_{k-1} + \theta x_{k-2}))(x_{k-1} - x_{k-2}) \right\|^2 d\theta \\
    \leq&\ \zeta^2 \mathbb{E}\|x_{k-1}-x_{k-2}\|^2. 
\end{align*}
Here, for the second inequality we used the mean value theorem. The last inequality holds from Assumption \ref{assump: similarity}. 

The third term can be bounded as follows:
\begin{align*}
    &\mathbb{E}\|g_k(x_{k-1}) - g_k(x_{k-2}) - \nabla f_{p}(x_{k-1}) + \nabla f_{p}(x_{k-2})\|^2\\
    =&\ \frac{1}{b}\mathbb{E}[\mathbb{E}_{z \sim D_{p}}\|\nabla \ell(x_{k-1}, z) - \nabla \ell(x_{k-2}, z) - \nabla f_{p}(x_{k-1}) + \nabla f_{p}(x_{k-2})\|^2] \\
    \leq&\ \frac{L^2}{b}\mathbb{E}\|x_{k-1} - x_{k-2}\|^2.
\end{align*}

Since $\xi_k \sim \mathcal N(0, \sigma_3^2 C_3^2\|x_{k-1} - x_{k-2}\|^2 I)$, we have $\mathbb{E}\|\xi_k\|^2 = C_3^2\sigma_3^2 \|x_{k-1} - x_{k-2}\|^2 d$ conditioned on iteration $k-1$.

Hence, we get
\begin{align*}
    \mathbb{E}\|\widetilde v_k - \nabla f(x_{k-1})\|^2 \leq&\ (1+1/K)\mathbb{E}\|\widetilde v_{k-1} - \nabla f(x_{k-2})\|^2 + \left((K+1)\zeta^2 + \frac{L^2}{b}\right)\mathbb{E}\|x_{k-1} - x_{k-2}\|^2 \\
    &+ C_3^2\sigma_3^2d\mathbb{E}\|x_{k-1}-x_{k-2}\|^2. 
\end{align*}

Recursively using this inequality results in 
\begin{align*}
    \mathbb{E}\|\widetilde v_k - \nabla f(x_{k-1})\|^2 \leq&\  (1+1/K)^{k-1}\mathbb{E}\|\widetilde v_{1} - \nabla f(x_{0})\|^2 \\
    &+ \left(K\zeta^2 + \frac{L^2}{b} + C_3^2\sigma_3^2d\right)\sum_{k'=2}^{k-1} (1+1/(K-1))^{k-1-k'}\mathbb{E}\|x_{k'-1} - x_{k'-2}\|^2 \\
    \leq&\ e\mathbb{E}\|\widetilde v_{1} - \nabla f(x_{0})\|^2 + e\left(K\zeta^2 + \frac{L^2}{b} + C_3^2\sigma_3^2d\right)\sum_{k=2}^{K}\mathbb{E}\|x_{k-1} - x_{k-2}\|^2. 
\end{align*}

Combined this inequality with (\ref{app_ineq: decent_lemma_sum}) results in

\begin{align*}
    \mathbb{E}\|\nabla f(x_{\hat k-1})\|^2 = \frac{1}{K}\sum_{k=1}^K \mathbb{E}\|\nabla f(x_{k-1})\|^2
    \leq&\ 
    \frac{2(\mathbb{E}[f(x_{0})] -  \mathbb{E}[f(x_K)])}{\eta K}  -  \frac{1}{2\eta^2}\frac{1}{K}\sum_{k=1}^K \mathbb{E}\|x_k - x_{k-1}\|^2  \\
    &+ 2e\mathbb{E}\|\widetilde v_1 - \nabla f(x_0)\|^2 \\
    &+  2e\left(K^2\zeta^2 + \frac{KL^2}{b} + KC_3^2\sigma_3^2d \right)\frac{1}{K}\sum_{k=1}^K\mathbb{E}\|x_k - x_{k-1}\|^2. 
\end{align*}

Suppose that $\eta \leq 1/(2\sqrt{2e(K^2\zeta^2 + KL^2/b + K\sigma_3^2C_3^2d)})$, which implies
\begin{align*}
    \frac{1}{4\eta^2} \geq 2e\left(K^2\zeta_2^2 + \frac{KL^2}{b} + KC_2^2\sigma_2^2d\right).
\end{align*}
Then, we have
\begin{align*}
    \mathbb{E}\|\nabla f(x_{\hat k-1})\|^2 \leq \frac{2(\mathbb{E}[f(x_0)] - \mathbb{E}[f(x_K)])}{\eta K} - \frac{1}{4\eta^2}\frac{1}{K}\sum_{k=1}^K \mathbb{E}\|x_k - x_{k-1}\|^2 +  2e\mathbb{E}\|\widetilde v_1 - \nabla f(x_0)\|^2.
\end{align*}
This is the desired result. 
\end{proof}

\begin{proposition}
\label{app_prop: convergence_rate}
Suppose that Assumptions \ref{assump: local_loss_smoothness}, \ref{assump: optimal_sol}, \ref{assump: local_loss_gradient_boundedness} and \ref{assump: similarity} hold. If $C_1 \geq G$ with $C_1 = \Theta(G)$, $C_2, C_3 \geq L$ with $C_2, C_3 = \Theta(L)$ and we appropriately choose $\eta = \Theta(\min\{1/L, 1/(K\zeta), \sqrt{b}/(\sqrt{K}L), 1/(K\sqrt{T}L\sigma_2\sqrt{d}), 1/(\sqrt{K}L\sigma_3 \sqrt{d}\})$, DIFF2-BVR-L-SGD satisfies
\begin{align*}
    \mathbb{E}\|\nabla f(x^\mathrm{out})\|^2 
    \leq&\  O\left(\frac{f(x_{0} - f(x_*)}{\eta KR} +  \sigma_1^2C_1^2d\right).
\end{align*}

\end{proposition}
\begin{proof}
From Proposition \ref{app_prop: local_opt}, we know that
\begin{align*}
    \mathbb{E}\|\nabla f(x_{r}^\mathrm{out})\|^2 \leq \frac{2(\mathbb{E}[f(x_{r-1})] - \mathbb{E}[f(x_{r})])}{\eta K} - \frac{1}{4\eta^2}\frac{1}{K}\sum_{k=1}^K \mathbb{E}\|x_{k, r-1} - x_{k-1, r}\|^2 + 2e\mathbb{E}\|\widetilde v_r - \nabla f(x_{r-1})\|^2.
\end{align*}

Averaging this inequality from $r=1$ to $R$, we get
\begin{align*}
    \mathbb{E}\|\nabla f(x^\mathrm{out})\|^2 = \frac{1}{R}\sum_{r=1}^R \mathbb{E}\|\nabla f(x_{\hat r-1})\|^2 \leq&\ \frac{2(f(x_0) - \mathbb{E}[f(x_*)])}{\eta K R} - \frac{1}{4\eta^2}\frac{1}{KR}\sum_{r=1}^R\sum_{k=1}^K \mathbb{E}\|x_{k, r-1} - x_{k-1, r}\|^2\\
    &+ \frac{2e}{R}\sum_{r=1}^R\mathbb{E}\|\widetilde v_r - \nabla f(x_{r-1})\|^2.
\end{align*}

We bound the last term. Similar to the arguments in the proof of 
Proposition \ref{app_prop: convergence_rate_DIFF2_gd_routine}, we can show that
\begin{align*}
    \frac{1}{R}\sum_{r=1}^R\left\|\widetilde v_r - \nabla f(x_{r-1})\right\|^2 
    \leq&\ \sigma_1^2C_1^2d + T\sigma_2^2C_2^2d\frac{1}{R}\sum_{r=1}^R \|x_{r-1} - x_{r-2}\|^2. 
\end{align*}

Moreover, since $\|x_r - x_{r-1}\|^2 \leq K\sum_{k=1}^K \|x_{k, r-1} - x_{k-1, r-1}\|^2$, we have

\begin{align*}
    \mathbb{E}\|\nabla f(x^\mathrm{out})\|^2 = \frac{1}{R}\sum_{r=1}^R \mathbb{E}\|\nabla f(x_{\hat r-1})\|^2 \leq&\ \frac{2(f(x_0) - f(x_*))}{\eta K R} - \frac{1}{4\eta^2}\frac{1}{KR}\sum_{r=1}^R\sum_{k=1}^K \mathbb{E}\|x_{k, r-1} - x_{k-1, r}\|^2\\
    &+ 2e\left(\sigma_1^2C_1^2d + K^2T\sigma_2^2C_2^2d\frac{1}{KR}\sum_{r=1}^R\sum_{k=1}^K\mathbb{E}\|x_{k, r-1} - x_{k-1, r}\|^2\right).
\end{align*}

Hence, under $\eta \leq 1/(2\sqrt{2eK^2T\sigma_2^2C_2^2d})$, which implies
\begin{align*}
    \frac{1}{4\eta^2} \geq 2eK^2T\sigma_2^2C_2^2d
\end{align*}
Using this fact and averaging the above inequality from $r=1$ to $R$, we obtain
\begin{align*}
    \mathbb{E}\|\nabla f(x^\mathrm{out})\|^2 
    \leq&\  \frac{2(f(x_{0} - f(x_*))}{\eta KR} + 2e \sigma_1^2C_1^2d.
\end{align*}
This is the desired result. 
\end{proof}

\begin{theorem}[Utility Bound]
\label{app_theorem: utility}
Suppose that Assumptions \ref{assump: local_loss_smoothness}, \ref{assump: optimal_sol}, \ref{assump: local_loss_gradient_boundedness} and \ref{assump: similarity} hold. Assume that $C_1 \geq G$ with $C_1 = \Theta(G)$, $C_2, C_3 \geq L$ with $C_2, C_3  = \Theta(L)$, $f(x_0) - f(x_*) = \Theta(1)$ and $n_\mathrm{min} P = \Omega( G^2\sqrt{d}/(L\varepsilon_\mathrm{DP}))$. Also, assume that the condition of $b$ in Proposition \ref{app_prop: dp_noise_level_DIFF2_BVR-L-SGD_routine} holds.  
Under the choices of $\sigma_1^2$, $\sigma_2^2$ and $\sigma_3^2$ in Proposition \ref{app_prop: dp_noise_level_DIFF2_BVR-L-SGD_routine}, if we appropriately choose $\eta = \Theta(\min\{1/L, 1/(K\zeta), \sqrt{b}/(\sqrt{K}L), 1/(K\sqrt{T}L\sigma_2\sqrt{d}), 1/(\sqrt{K}L\sigma_3 \sqrt{d}\})$ and $T = \Theta(\max\{1,\tau R\})$ with $\tau := (G^2\sqrt{d}/(Ln_\mathrm{min}P\varepsilon_\mathrm{DP}))^{2/3}$, DIFF2-BVR-L-SGD satisfies
\begin{align*}
    \mathbb{E}\|\nabla f(\widetilde x^\mathrm{out})\|^2  \leq&\ 
    \left(O\left(\frac{L}{K} + \zeta + \frac{L}{\sqrt{K}b}\right) + \widetilde O\left(\frac{L\sqrt{d}}{\sqrt{K}b\sqrt{\varepsilon_\mathrm{DP}}}\right)\right)\frac{1}{R} 
    + \widetilde O\left(\frac{L\sqrt{d}}{n_\mathrm{min}\sqrt{P}\varepsilon_\mathrm{DP}}\right)\frac{1}{\sqrt{R}}
    + \widetilde O\left(\frac{(LGd)^\frac{2}{3}}{(n_\mathrm{min}P\varepsilon_\mathrm{DP})^\frac{4}{3}}
    \right)
\end{align*}
In particular, setting 
\begin{align*}
    R = 1 + \left(\Theta\left(\frac{L}{K} + \zeta + \frac{L}{\sqrt{K}b}\right) + \widetilde \Theta\left( \frac{L\sqrt{d}}{\sqrt{K}b\sqrt{\varepsilon_\mathrm{DP}}}\right)\right)\frac{1}{\varepsilon_\mathrm{opt}} + \widetilde \Theta\left(\frac{L^2d}{n_\mathrm{min}^2P\varepsilon_\mathrm{DP}^2}\right)\frac{1}{\varepsilon_\mathrm{opt}^2}
\end{align*}
results in utility
\begin{align*}
    \mathbb{E}\|\nabla f(x^\mathrm{out})\|^2 \leq \varepsilon_\mathrm{opt} := \widetilde \Theta\left(\frac{(LGd)^\frac{2}{3}}{(n_\mathrm{min}P\varepsilon_\mathrm{DP})^\frac{4}{3}}\right).
\end{align*}
\end{theorem}

\begin{proof}
From Proposition \ref{app_prop: convergence_rate}, we know that 
\begin{align}
    \mathbb{E}\|\nabla f(\widetilde x^\mathrm{out})\|^2 \leq&\ O\left( \frac{1}{\eta KR} + G^2\sigma_1^2d\right). \label{app_ineq: convergence_rate}
\end{align}
since $C_1 = \Theta(G)$.

By Proposition \ref{app_prop: dp_noise_level_DIFF2_BVR-L-SGD_routine}, the necessary noise levels for $(\varepsilon_\mathrm{DP}, \delta_\mathrm{DP})$-DP with respect to $D_p$ are 
$\sigma_1^2 = \widetilde \Theta(R/(T n_\mathrm{min}^2P^2 \varepsilon_\mathrm{DP}^2))$, $\sigma_2^2 = \widetilde \Theta(R/(n_\mathrm{min}^2P^2\varepsilon_\mathrm{DP}^2))$ and $\sigma_3^2 =  KR/(n_\mathrm{min}^2P\varepsilon_\mathrm{DP}^2) + 1/(b^2\varepsilon_\mathrm{DP})$ 
respectively. Substituting the definition of $\sigma_1^2$ to (\ref{app_ineq: convergence_rate}), we have
\begin{align*}
    \mathbb{E}\|\nabla f(\widetilde x^\mathrm{out})\|^2 \leq&\  O\left( \frac{1}{\eta KR}\right) + \widetilde O\left(\frac{RG^2d}{Tn_\mathrm{min}^2P^2\varepsilon_\mathrm{DP}^2}\right). 
\end{align*}
Substituting the definition of $\eta$, $\sigma_2^2$, $\sigma_3^2$ to the obtained result, it holds that 
\begin{align*}
    &\mathbb{E}\|\nabla f(\widetilde x^\mathrm{out})\|^2 \\ \leq&\ O\left( \frac{L + K\zeta + \frac{\sqrt{K}}{\sqrt{b}}L}{KR}\right) + \widetilde O\left( \frac{K\sqrt{T}L\sigma_2\sqrt{d} + \sqrt{K}L\sigma_3\sqrt{d}}{KR} + \frac{RG^2d}{Tn_\mathrm{min}^2P^2\varepsilon_\mathrm{DP}^2}\right) \\
    =&\ O\left( \frac{\frac{L}{K} + \zeta + \frac{L}{\sqrt{Kb}}}{R}\right) + \widetilde O\left(\frac{\sqrt{T}L\sigma_2\sqrt{d} + \frac{L\sigma_3\sqrt{d}}{\sqrt{K}}}{R} + \frac{RG^2d}{Tn_\mathrm{min}^2P^2\varepsilon_\mathrm{DP}^2}\right) \\
    =&\  O\left( \frac{\frac{L}{K} + \zeta + \frac{L}{\sqrt{Kb}}}{R}\right) + \widetilde O\left( \frac{\frac{\sqrt{TR}L\sqrt{d}}{n_\mathrm{min}P\varepsilon_\mathrm{DP}} + \frac{L\sqrt{R}\sqrt{d}}{n_\mathrm{min}\sqrt{P}\varepsilon_\mathrm{DP}} + \frac{L\sqrt{d}}{\sqrt{K}b\sqrt{\varepsilon_\mathrm{DP}}}}{R} + \frac{RG^2d}{Tn_\mathrm{min}^2P^2\varepsilon_\mathrm{DP}^2}\right) \\
    =&\ \left(O\left(\frac{L}{K} + \zeta + \frac{L}{\sqrt{K}b}\right) + \widetilde O\left( \frac{L\sqrt{d}}{\sqrt{K}b\sqrt{\varepsilon_\mathrm{DP}}}\right)\right)\frac{1}{R} 
    + \widetilde O\left(\frac{\sqrt{T}L\sqrt{d}}{n_\mathrm{min}P\varepsilon_\mathrm{DP}} 
    + \frac{L\sqrt{d}}{n_\mathrm{min}\sqrt{P}\varepsilon_\mathrm{DP}}\right)\frac{1}{\sqrt{R}}
    + \widetilde O\left( \frac{RG^2d}{Tn_\mathrm{min}^2P^2\varepsilon_\mathrm{DP}^2}
    \right)
\end{align*}

Assume that $n_\mathrm{min} P = \Omega( G^2\sqrt{d}/(L\varepsilon_\mathrm{DP}))$.
We define
\begin{align*}
    T := \Theta\left( 1 \vee  \left(\frac{G^2\sqrt{d}}{Ln_\mathrm{min}P\varepsilon_\mathrm{DP}}\right)^\frac{2}{3} R\right)
\end{align*}
with $1 \leq T \leq R$.

Then, we have
\begin{align*}
    \mathbb{E}\|\nabla f(\widetilde x^\mathrm{out})\|^2  \leq&\ 
    \left(O\left(\frac{L}{K} + \zeta + \frac{L}{\sqrt{K}b}\right) + \widetilde O\left(\frac{L\sqrt{d}}{\sqrt{K}b\sqrt{\varepsilon_\mathrm{DP}}}\right)\right)\frac{1}{R} 
    + \widetilde O\left(\frac{L\sqrt{d}}{n_\mathrm{min}\sqrt{P}\varepsilon_\mathrm{DP}}\right)\frac{1}{\sqrt{R}}
    + \widetilde O\left(\frac{(LGd)^\frac{2}{3}}{(n_\mathrm{min}P\varepsilon_\mathrm{DP})^\frac{4}{3}}
    \right)
\end{align*}

Hence, setting
\begin{align*}
    R = 1 + \left(\Theta\left(\frac{L}{K} + \zeta + \frac{L}{\sqrt{K}b}\right) + \widetilde \Theta\left( \frac{L\sqrt{d}}{\sqrt{K}b\sqrt{\varepsilon_\mathrm{DP}}}\right)\right)\frac{1}{\varepsilon_\mathrm{opt}} + \widetilde \Theta\left(\frac{L^2d}{n_\mathrm{min}^2P\varepsilon_\mathrm{DP}^2}\right)\frac{1}{\varepsilon_\mathrm{opt}^2}
\end{align*}
with $\varepsilon_\mathrm{opt} := \widetilde \Theta\left(\frac{(LGd)^\frac{2}{3}}{(n_\mathrm{min}P\varepsilon_\mathrm{DP})^\frac{4}{3}}\right) $ 
gives the desired result. 
\end{proof}

\end{document}